%% file: example_paper.tex
\documentclass{article}

\usepackage{microtype}
\usepackage{graphicx}
\usepackage{subcaption}
\usepackage{booktabs} 
\usepackage[dvipsnames]{xcolor}
\usepackage{comment}
\usepackage{hyperref}

\usepackage[accepted]{icml2026}

\usepackage{amsmath}
\usepackage{amssymb}
\usepackage{mathtools}
\usepackage{amsthm}

\usepackage{pifont}
\usepackage[capitalize,noabbrev]{cleveref}

\theoremstyle{plain}
\newtheorem{theorem}{Theorem}[section]
\newtheorem{proposition}[theorem]{Proposition}
\newtheorem{lemma}[theorem]{Lemma}
\newtheorem{corollary}[theorem]{Corollary}
\theoremstyle{definition}

\theoremstyle{remark}

\newcommand{\myhighlight}[1]{\colorbox{cyan!15}{\parbox{\dimexpr\linewidth-2\fboxsep}{\strut #1\strut}}}
\usepackage[textsize=tiny]{todonotes}

\AtBeginDocument{%
  \setcitestyle{notesep={; }}%
}

\usepackage{enumerate}
\usepackage{pdfsync}
\usepackage[OT1]{fontenc}
\usepackage{natbib}
\usepackage{macro_commands}
\usepackage{pbox}
\usepackage{setspace}
\usepackage{tabularx}
\usepackage{float}
\usepackage{wrapfig,lipsum}
\usepackage{enumitem}
\usepackage{graphicx}
\usepackage{pgfplots}
\usepackage{mathtools}
\usepackage{caption}
\usetikzlibrary{arrows, shapes,snakes, automata,backgrounds,petri}
\usepackage{booktabs} 

\usepackage{graphicx} 
\usepackage{pdfpages} 
\usepackage{hyperref}

\usepackage{color, colortbl}
\usepackage{xcolor}
\definecolor{LightCyan}{rgb}{0.8, 0.9, 1}

\begin{document}

\twocolumn[
  \icmltitle{\textit{S}-SPPO: Semantic-Calibrated Self-Play Preference Optimization}



  \icmlsetsymbol{equal}{*}

  \begin{icmlauthorlist}
    \icmlauthor{Xiwen Chen}{comp,equal}
    \icmlauthor{Wenhui Zhu}{asu,equal}
    \icmlauthor{Jingjing Wang}{sch,equal}
    \icmlauthor{Peijie Qiu}{was,equal}
    \icmlauthor{Zhipeng Wang}{Rice}
    \icmlauthor{Huayu Li}{az}
    \icmlauthor{ZhengXiao He}{az}
    \icmlauthor{Xuanzhao Dong}{asu}
    \icmlauthor{Prayag Tiwari}{Halmstad}
    \icmlauthor{Mingkun Xu}{Guangdong}
    \icmlauthor{Yujian Xiong}{asu}
    \icmlauthor{Feng Luo}{sch}
    \icmlauthor{Abolfazl Razi}{sch}
    \icmlauthor{Brendan Hogan Rappazzo}{comp}
    \icmlauthor{Anderson Schneider}{comp}
    \icmlauthor{Yuriy Nevmyvaka}{comp}
\end{icmlauthorlist}

  \icmlaffiliation{az}{University of Arizona, USA}
  \icmlaffiliation{asu}{Arizona State University, USA}
  \icmlaffiliation{comp}{Morgan Stanley, USA}
  \icmlaffiliation{Rice}{Now at Google LLC, work done at Rice University}
  \icmlaffiliation{sch}{Clemson University, USA}
  \icmlaffiliation{was}{Washington University in St. Louis, USA}
  \icmlaffiliation{Halmstad}{Halmstad University, Sweden}
  \icmlaffiliation{Guangdong}{Guangdong Institute of Intelligence Science and Technology, China}


  \icmlcorrespondingauthor{Xiwen Chen}{{xiwenchen.cs@outlook.com}}
  \icmlcorrespondingauthor{Wenhui Zhu}{{wzhu59@asu.com}}

  \icmlkeywords{Machine Learning, ICML}

  \vskip 0.3in
]

\printAffiliationsAndNotice{}  

\begin{abstract}
Aligning Large Language Models (LLMs) with human preferences is often formulated via Direct Preference Optimization (DPO). However, the standard Bradley-Terry instantiation of DPO is limited in modeling common departures from transitivity in human preferences. To address this, recent work has introduced Self-Play Preference Optimization (SPPO), which iteratively refines the policy by training on self-generated win-lose pairs. Our investigation, however, reveals a critical instability in SPPO: the optimization is prone to \textit{policy degeneration} when the preference oracle assigns overly confident wins to semantically indistinguishable responses. To mitigate this, we propose \textit{S}-SPPO, a dual-space semantic calibration framework comprising: i) \textit{Supervision Calibration} via semantic gating, which anneals win rate targets toward the maximum-entropy baseline as semantic overlap increases; and ii) \textit{Representation Calibration} via latent repulsion to enforce geometric diversity to prevent manifold collapse and maintain latent diversity between chosen and rejected samples. Theoretically, we show that the calibration preserves the constant-sum game structure, facilitating convergence to a Nash Equilibrium. Empirically, \textit{S}-SPPO avoids the performance degradation seen in prior methods, achieving \textbf{52.19\% win rate} and \textbf{47.46\% length-controlled win rate} on AlpacaEval 2.0 with Llama-3-8B, without using additional human-annotated preferences during training. The code will be available at \url{https://github.com/xiwenc1/s-sppo}.
\end{abstract}

\section{Introduction}
\begin{figure}[!t]
    \centering
    \includegraphics[width=1.0\linewidth]{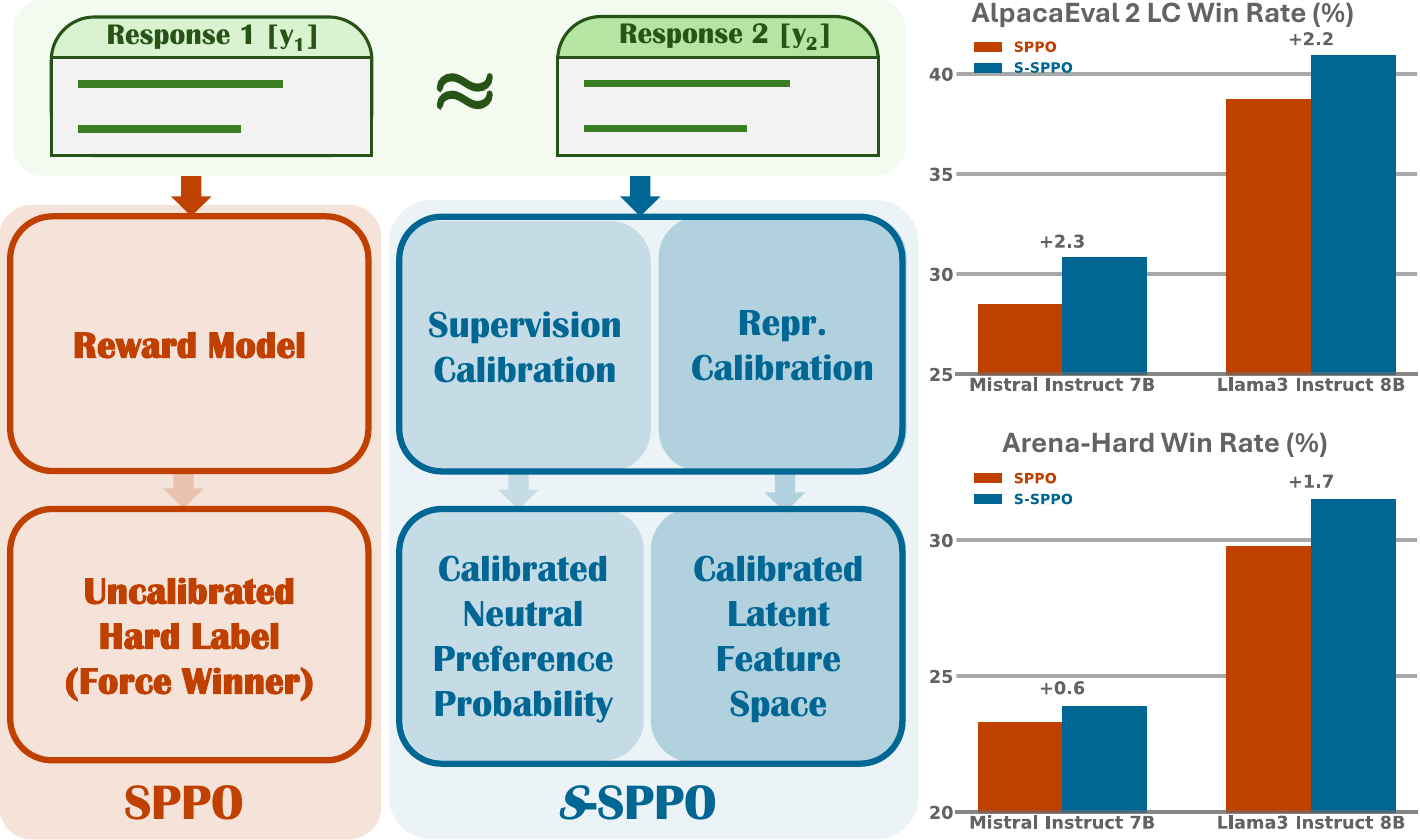}
    \caption{Comparison between \textcolor{MidnightBlue}{\textit{S}-SPPO} and \textcolor{Mahogany!80}{SPPO}. When responses ($\yb_1, \yb_2$) are semantically similar, SPPO imposes uncalibrated hard preferences that can be arbitrary, whereas S-SPPO employs \textcolor{MidnightBlue!84}{Supervision} and \textcolor{MidnightBlue!76}{Representation (Repr.)} calibration to produce calibrated preference probabilities (left), resulting in more stable training and consistent empirical improvements (right).}
    \label{fig:teaser}
\end{figure}

The alignment of Large Language Models (LLMs) with human values has become a cornerstone of modern AI development.~\cite{ouyang2022training,bai2022training,touvron2023llama,sun2024aligning,chen2025aha}
The standard paradigm, Reinforcement Learning from Human Feedback~\citep[RLHF,][]{christiano2017deep,ouyang2022training}, typically involves training a reward model to proxy human preferences, followed by optimizing the policy via Proximal Policy Optimization~\citep[PPO,][]{schulman2017proximal}. 
While effective, this pipeline is often criticized for its complexity, training instability, and sensitivity to hyperparameter tuning~\cite{engstrom2020implementation,azar2024general,casper2023open,rafailov2023direct}.

To mitigate these challenges, Direct Preference Optimization~\citep[DPO,][]{rafailov2023direct} emerged as a compelling alternative. 
By reparameterizing the optimal policy in closed form, DPO bypasses the need for an explicit reward model, treating preferences as an implicit reward signal to be optimized directly. 
However, the standard DPO formulation relies on the \textit{Bradley-Terry (BT) model}~\cite{bradley1952rank}, which assumes that preferences satisfy strict transitivity and can be fully captured by a scalar reward function. 
This assumption, however, often fails to account for the \textit{intransitivity} and irrationality inherent in general human preferences~\citep{wu2024self}.

To bridge this gap, Self-Play Preference Optimization~\citep[SPPO,][]{wu2024self} reframes the alignment problem through the lens of game theory. 
Instead of maximizing a static reward, SPPO formulates preference optimization as a two-player constant-sum game between the policy and the environment. 
By approximating the Nash Equilibrium, SPPO theoretically enables the model to handle non-transitive preferences robustly, without being constrained by the limitations of the Bradley-Terry model.

However, SPPO implicitly assumes a well-calibrated preference oracle. Specifically, the algorithm treats estimated win rates as precise regression targets for policy optimization. This assumption breaks down in practice. We observe that practical reward models are often highly overconfident, assigning extreme win probabilities (approaching 0 or 1) even when candidate responses are semantically indistinguishable. \textit{This introduces a critical misalignment in the optimization landscape.} Since SPPO forces the model to fit these overconfident targets, the optimizer effectively tries to minimize the residual between the implicit reward (\textit{i.e.,} log-likelihood ratio) and noise in the preference predictions. Consequently, the training signal becomes dominated by these \textit{false positives} on semantically indistinguishable pairs, aggressively pushing the model to distinguish between highly similar outputs. This artificial scaling of gradient variance destabilizes the self-play trajectory, preventing convergence to a meaningful Nash Equilibrium.

To bridge this gap, we propose Semantic-Calibrated Self-Play Preference Optimization ({\textit{S}-SPPO), a dual-space semantic calibration framework designed to stabilize self-play dynamics. 
Conceptually, we reframe preference optimization as a \textit{Dual-Space Min-Max game}. We identify that stable alignment requires satisfying adversarial constraints: \textit{minimizing} preference divergence in the label space to prevent overfitting noise, while simultaneously \textit{maximizing} representation distance in the latent space to prevent manifold collapse (see Fig.~\ref{fig:teaser}(a)). This leads to consistent performance improvement in various benchmarks (Fig.~\ref{fig:teaser}(b)). 

To summarize, our main contributions are as follows:
\begin{itemize}
    \item  We uncover a critical instability in self-play preference optimization (SPPO), revealing that the optimization is prone to learning instability and performance degradation when the preference oracle assigns overly confident wins to semantically indistinguishable responses.
    
    \item We propose \textit{S}-SPPO, which introduces two novel calibration mechanisms: \textit{i)} \textit{Supervision Calibration} via semantic gating to anneal win rate targets based on semantic overlap; and \textit{ii)} \textit{Representation Calibration}, which applies a repulsive force in the latent space to prevent mode collapse.
    
    \item  We provide a theoretical analysis proving that our semantic calibration preserves the constant-sum game structure, thereby facilitating convergence to a Nash Equilibrium comparable to the original SPPO formulation.
    
   \item Empirically, \textit{S}-SPPO significantly outperforms existing baselines. On Llama-3-8B, it achieves a 52.19\% win rate and a 47.46\% length-controlled win rate on AlpacaEval 2.0 without relying on {additional} human-annotated preference data or supervision from stronger models (e.g., GPT-4) during training, effectively eliminating the performance degradation observed in standard self-play methods.
\end{itemize}

\begin{figure*}[!t]
    \centering
    \includegraphics[width=0.9\linewidth]{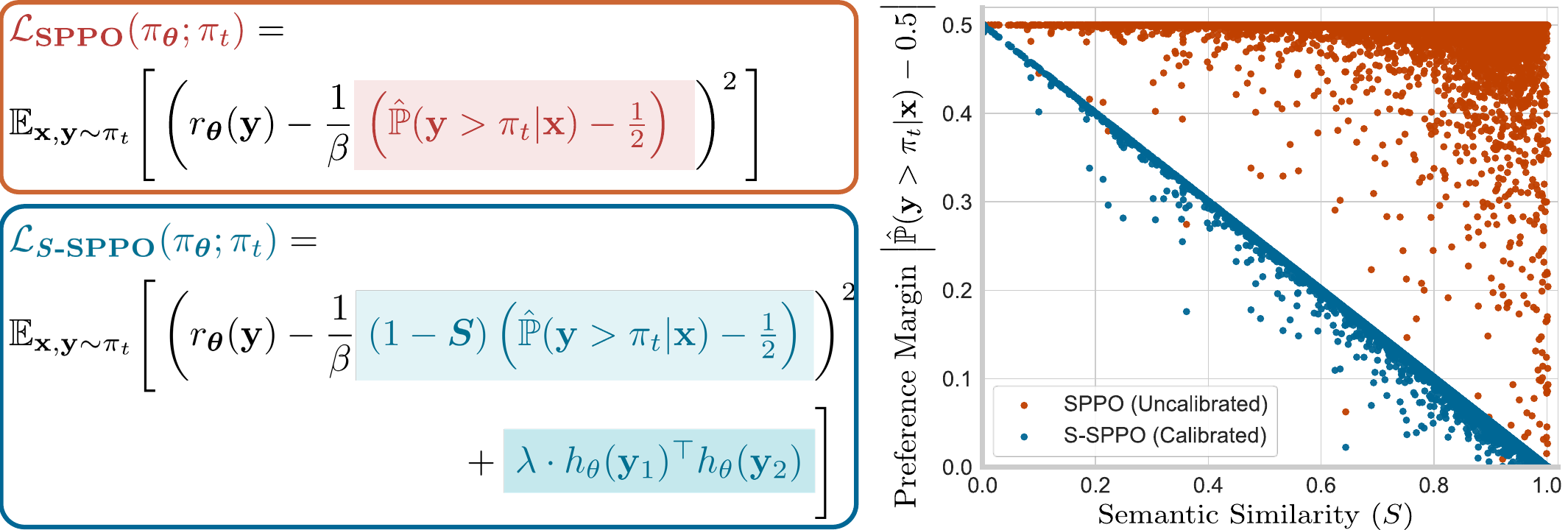}
    \caption{ \textcolor{MidnightBlue}{\textit{S}-SPPO} and \textcolor{Mahogany!80}{SPPO} mainly differ in their preference optimization objectives, as indicated in the shaded boxes (left). The preference margin $|\hat{\PP}(\yb \succ \pi_t | \xb) - 0.5 |$ in \textcolor{Mahogany!80}{SPPO} is uncalibrated, whereas \textcolor{MidnightBlue}{\textit{S}-SPPO} calibrates the preference margin by the semantic similarity $(S)$ between win-lose pairs (right). See Appendix~\ref{app:aba-4} for more details.}
    \label{fig:overview}
\end{figure*}

\section{Background and Limitations of SPPO}
\label{sec:preliminaries}


\subsection{Preference Optimization Landscape}
RLHF~\citep[][]{ouyang2022training} has emerged as a promising paradigm for aligning LLMs with human preference. Traditional RLHF approaches~\citep[\textit{e.g.,}][]{schulman2017proximal} typically follow a two-stage process. Given a prompt $\xb \sim \mathcal{X}$, where $\mathcal{X}$ denotes the prompt distribution, an LLM policy $\pi(\cdot \mid \xb)$ generates a response $\yb \in \mathcal{Y}$. The model aims to update the policy $\pi$ to better match the human-preferred response distribution. The optimization objective is usually formulated as maximizing the expected reward subject to a Kullback–Leibler (KL) divergence constraint:
\begin{equation}
\label{eqn:rlhf_objective}
\begin{aligned}\nonumber
    \EE_{\substack{\xb \sim \cX \\ \yb \sim \pi(\cdot|\xb)}} [r_\phi(\yb; \xb)] - \beta \EE_{\xb \sim \cX} [\mathrm{KL}(\pi(\cdot|\xb) \| \pi_{\text{ref}}(\cdot|\xb))],
\end{aligned}
\end{equation}
where $\pi_{\text{ref}}$ is the reference policy (typically the Supervised Fine-Tuned model) and $\beta$ controls the deviation penalty. While effective, traditional RLHF methods relying on explicit reward modeling often overlook the inherent semantic ambiguity in language.

\textbf{Direct Preference Optimization.}
DPO~\citep[][]{rafailov2023direct} bypasses explicit reward modeling by treating preferences as an implicit reward signal, substantially simplifying RLHF for preference optimization. As is standard, we assume access to an offline preference dataset $\mathcal{D}=\{(\xb, \yb_1, \yb_2)\}$ comprising preferred response $\yb_1$ and dispreferred response $\yb_2$ for prompt $\xb$.
DPO maximizes the log-likelihood $\PP(\yb_1 \succ \yb_2 | \xb)$ under the Bradley-Terry (BT) preference model~\citep[][]{bradley1952rank}.
However, the off-policy DPO is known to fall short in capturing the intransitivity and irrationality in human preferences~\cite{wu2024self}.

\textbf{Self-Play Preference Optimization.}
SPPO~\citep{wu2024self} mitigates DPO’s limitations by optimizing preference probabilities directly in a self-play setting. It recasts preference optimization as a two-player constant-sum game between a policy $\pi$ and an opponent $\pi'$, defining the expected win rate as
\begin{equation}
\label{eqn:policy_win_rate}\nonumber
    \PP(\pi \succ \pi' | \xb) \triangleq \EE_{\substack{\yb \sim \pi(\cdot|\xb) \\ \yb' \sim \pi'(\cdot|\xb)}} \big[ \PP(\yb \succ \yb' | \xb) \big].
\end{equation}
The policy targets the \textit{von Neumann winner} $\pi^*$, which is the Min–Max optimal strategy maximizing the win rate against any adversary:
\begin{equation}
\label{eqn:nash_eq}\nonumber
    (\pi^*, \pi^*) = \arg \max_{\pi} \min_{\pi'} \EE_{\xb \sim \cX} [\PP(\pi \succ \pi' | \xb)].
\end{equation}
To approximate $\pi^*$, SPPO iteratively updates the policy. At iteration $t$, an online preference dataset $\mathcal{D}_t$ is constructed by generating $K$ responses and annotating win-lose pairs using an off-the-shelf reward model~\citep[\textit{e.g.,} PairRM,][]{jiang2023llm}.
Rather than maximizing a contrastive margin, SPPO treats the empirical win rate as a regression target. It updates the policy by minimizing the $L_2$ distance between the implicit reward $r_{\btheta}(\yb)$ (log-likelihood ratio) and the shifted win rate:
\begin{equation}
\label{eqn:sppo_loss}
\begin{split}
    \cL_{\textbf{SPPO}}(\pi_{\btheta}; \pi_t) & = \EE_{\mathcal{D}_t} \Bigg[ \Bigg( \underbrace{\log \frac{\pi_{\btheta}(\yb|\xb)}{\pi_t(\yb|\xb)}}_{r_{\btheta}(\yb)} \\
    & \quad \quad - \frac{1}{\beta} \Big( \hat{\PP}(\yb \succ \pi_t | \xb) - \frac{1}{2} \Big) \Bigg)^2 \Bigg].
\end{split}
\end{equation}
This objective effectively regresses the policy's log-ratio towards the preference probability, pushing the likelihood of winning responses away from losing ones to approximate the Nash Equilibrium dynamics.

\subsection{Limitations of SPPO}
The success of SPPO hinges on the reliability of the preference supervision signals $\mathcal{D}_t$ used for policy update at every iteration. 
Ideally, the preference oracle should reflect the underlying \textit{epistemic uncertainty}: when two candidate responses are semantically indistinguishable, the estimated win rate should naturally converge to the neutral point $0.5$. In other words, the \textit{preference margin} $m(\yb) = |\hat{\PP}(\yb \succ \pi_t | \xb) - 0.5 |$ should approach zero. To reveal this, we empirically investigate the relationship between the semantic similarity between win-lose pairs and the preference margin.
However, we uncover a critical misalignment between the semantic similarity ($S$) of win–lose pairs and the reward model’s induced preferences (Fig.~\ref{fig:overview}(right)). In particular, the reward model (e.g., PairRM) is poorly calibrated: standard SPPO continues to produce large preference margins even when the two responses are semantically indistinguishable (e.g., $S>0.9$). 
As a result, the preference signal of SPPO is not aligned with semantic similarity, as highly similar responses can still be assigned extreme win probabilities, causing the optimization to treat negligible differences as substantive and to apply overly strong policy updates.
This behavior indicates that discrete preference supervision fails to respect the inherently continuous structure of semantic similarity, effectively imputing substantial quality gaps where none exist. Consequently, the SPPO training signal can become decoupled from semantic consistency, degrading learning dynamics: to match high-confidence targets on ambiguous pairs, the policy may overfit to spurious cues or reward-model noise. Moreover, magnifying these high-magnitude errors can increase gradient variance, destabilize optimization, and hinder convergence toward a semantically consistent Nash Equilibrium. The gradient analysis of SPPO reveals this fact
\begin{equation}
    \nabla_{\btheta} \cL_{\textbf{SPPO}}(\pi_{\btheta}; \pi_t) = \underbrace{2\left(r_{\btheta}(\yb) - \frac{ m(\yb)}{\beta}\right)}_{w(\yb)} \cdot \nabla_{\btheta} \log \pi_{\btheta}(\yb|\xb),
\end{equation}
where the term $\nabla_{\btheta} \log \pi_{\btheta}(\yb|\xb)$ increases the log-probability of the chosen response, and the magnitude of $w(\yb)$ determines how hard SPPO pushes towards that direction.
This analysis highlights the potential for performance degradation, which is consistent with the downward trends observed in Table~\ref{tab:mt-bench} and Table~\ref{tab:arena-hard}.

\section{Method: Dual-Space Semantic Calibration}
\label{sec:method}


Building on the analysis above, we propose \textit{S}-SPPO, a framework that casts preference optimization as a \textit{Dual-Space Min–Max game}. The framework enforces adversarial constraints at two complementary levels: \textit{i)} Minimization: in the label space, Supervision Calibration minimizes preference divergence on ambiguous pairs, mitigating supervision bias; and \textit{ii)} Maximization: in the latent space, Representation Calibration enforces geometric diversity between responses, preventing manifold collapse. Together, this push-pull dynamic stabilizes the policy under noisy supervision while preserving the expressiveness required for better alignment.

\subsection{Supervision Calibration via Semantic Gating}
\label{sec:sup_calibration}
The raw empirical win rate $\hat{\PP}$, computed from reward-model scores, is often poorly calibrated with respect to semantic utility. To mitigate this, we introduce a semantic gating mechanism. Let $\phi(\cdot)$ denote a model-agnostic semantic encoder~\citep[][]{wang2020minilm,gunther2023jina,nussbaum2024nomic}, and define the \textit{rectified semantic similarity} between responses $\yb_i$ and $\yb_j$ as follows:
\vspace{-0.05in}
\begin{equation}\nonumber
    S_{ij} = \max \left( 0, \frac{\phi(\yb_i)^\top \phi(\yb_j)}{\| \phi(\yb_i) \| \| \phi(\yb_j) \|} \right).
\end{equation}
We then obtain a calibrated target probability $\hat{\PP}_c$ by annealing $\hat{\PP}$ toward the maximum-entropy baseline (i.e., $0.5$)\footnote{Following~\citet{wu2024self}, we set the baseline target to $0.5$. While originally derived as a normalization under random preferences, in our semantic calibration setting it naturally corresponds to the maximum-entropy case with no preference between semantically indistinguishable responses.}, in proportion to each candidate’s average semantic overlap with the other responses:
\vspace{-0.05in}
\begin{equation}
\label{eqn:calibration_formula}
\begin{aligned}
    \hat{\PP}_c(\yb_i \succ \pi_t | \xb) = \frac{1}{2} &+ \bigg( 1 - \frac{1}{K-1} \sum_{j \neq i} S_{ij} \bigg) \\
    & \quad \quad \times \bigg( \hat{\PP}(\yb_i \succ \pi_t | \xb) - \frac{1}{2} \bigg).
\end{aligned}
\end{equation}
This formulation ensures the supervision signal is calibrated to semantic similarity: gradients are masked ($ \hat{\PP}_c \to 0.5$) when a response is semantically indistinguishable from its competitors, preventing the policy from fitting label noise.

\noindent\textbf{Practical Implementation.} 
While Eq.~\eqref{eqn:calibration_formula} provides a general formulation considering the semantic density over $K$ samples, computing the full similarity matrix can be computationally expensive for large $K$. In our practical implementation (Algorithm~\ref{alg:s_sppo}), we align with the standard SPPO training paradigm, which samples a winner-loser pair $(\yb_w, \yb_l)$ for optimization. 
Accordingly, we approximate the average semantic overlap using the specific similarity between the selected pair, i.e., $S \approx S_{wl}$. 
This serves as an efficient Monte Carlo estimate of the semantic ambiguity, ensuring that the calibrated target $\hat{\PP}_c$ effectively penalizes overconfidence when the chosen pair is semantically ambiguous.

\noindent\textbf{Connection to Tilted-ERM.}
Our approach relates to Tilted-ERM~\citep{li2020tilted} in suppressing outliers. However, unlike TERM which performs blind re-weighting based on loss magnitude, \textit{S}-SPPO utilizes semantic similarity as an informative prior. This selectively dampens signals for ambiguous pairs while retaining supervision for semantically distinct yet challenging instances.

\subsection{Representation Calibration via Latent Repulsion}
\label{sec:rep_calibration}

While Supervision Calibration mitigates label noise, frequent annealing of target probabilities to 0.5 introduces a regularization trade-off: the policy may lose discriminative power. This creates a risk of \textit{manifold collapse}, where distinct generation trajectories map to near-identical latent representations due to gradient starvation.
We propose Representation Calibration to enforce geometric diversity. Even for semantically ambiguous pairs, the model must maintain distinct internal states for different trajectories. Let $h_{\btheta}(\yb)$ be the normalized final hidden state of an LLM. We define the calibration term on sampled pairs $(\yb_1, \yb_2)$ as:
\begin{equation}\label{eqn:repr}
    \cL_{\text{Rep}}(\btheta) = \EE_{(\xb, \yb_1, \yb_2) \sim \cD_t} \Big[ h_{\btheta}(\yb_1)^\top h_{\btheta}(\yb_2) \Big].
\end{equation}
Minimizing this term applies a repulsive force between the feature vectors of the chosen and rejected responses. Geometrically, this acts as a soft repulsion constraint, preventing the latent representations of distinct responses from collapsing into a single point, thereby maintaining the necessary volume in the embedding space for effective learning.

\noindent\textbf{Final Objective.}
The dual-space \textit{S}-SPPO objective at the $t$-th iteration is a weighted sum of (\ref{eqn:calibration_formula}) and (\ref{eqn:repr}):
\begin{equation}
\label{eqn:final_objective}\nonumber
\begin{split}
    & \cL_{\textbf{\textit{S}-SPPO}}(\pi_{\btheta}; \pi_t) = \EE_{\mathcal{D}_t} \Bigg[  \Bigg( r_{\btheta}(\yb) -   \\  &\quad \quad \frac{1}{\beta} \left( \hat{\PP}_c(\yb \succ \pi_t | \xb) - \frac{1}{2} \right) \Bigg)^2   + \lambda \cdot  h_{\btheta}(\yb_1)^\top h_{\btheta}(\yb_2)  \Bigg].
\end{split}
\end{equation}
where $\lambda$ is a weight balance parameter.

\paragraph{Rationale: Decoupling Calibration Sources.}
Crucially, our framework employs distinct metric sources to match the objective of each space. For Supervision Calibration, we utilize an \textit{external} embedding model~\citep{wang2020minilm,gunther2023jina,nussbaum2024nomic} to establish a {model-agnostic anchor}, ensuring that semantic targets remain objective and stable throughout training. Conversely, Representation Calibration operates directly on the \textit{policy's own} hidden states. Since manifold collapse is an {intrinsic degeneration} of the feature space, we explicitly intervene within the policy's latent geometry to enforce diversity where it matters most.

\subsection{Theoretical Analysis}
\label{sec:theory}

As discussed in Section~\ref{sec:preliminaries}, standard SPPO formulates preference optimization as a symmetric two-player constant-sum game at the \emph{policy level}, where the payoff is defined by the expected win rate $\PP(\pi \succ \pi' \mid \xb)$. Its convergence guarantees rely critically on the constant-sum structure (i.e., antisymmetry) of this game. \textit{S}-SPPO modifies the empirical preference signal through semantic calibration. A natural theoretical concern is whether this calibration breaks the constant-sum property, potentially destabilizing the Nash Equilibrium. In this section, we rigorously prove that our semantic calibration mechanism preserves this essential game-theoretic structure, thereby suggesting that \textit{S}-SPPO inherits the convergence guarantees of standard SPPO.

\subsubsection{Policy-Level Constant-Sum Preservation}

We analyze the properties of the calibrated preference probability $\hat{\PP}_c$ in the general pairwise setting.

\begin{lemma}[Policy-Level Constant-Sum Preservation]
\label{lem:policy_constant_sum}
Consider the pairwise calibrated preference probability $\hat{\PP}_c$ consistent with the update rule in Eq.~\eqref{eqn:calibration_formula}. For any pair of policies $(\pi, \pi')$ and any prompt $\xb$, the induced policy-level win rates satisfy:
\begin{equation}
\PP_c(\pi \succ \pi' \mid \xb) + \PP_c(\pi' \succ \pi \mid \xb) = 1,
\end{equation}
where $\PP_c(\pi \succ \pi' \mid \xb) \triangleq \EE_{\yb \sim \pi, \yb' \sim \pi'} [\hat{\PP}_c(\yb \succ \yb' \mid \xb)]$.
Consequently, semantic calibration induces a valid symmetric constant-sum game at the policy level.
\end{lemma}
\begin{proof}
    See Appendix~\ref{app:proof1}.
\end{proof}

\begin{corollary}[Convergence of the Calibrated Game]
\label{cor:convergence}
Let $\bar{\pi}_T = \frac{1}{T} \sum_{t=1}^T \pi_t$ denote the averaged policy produced by \textit{S}-SPPO (considering the calibrated objective). Under standard realizability assumptions, the duality gap of $\bar{\pi}_T$ with respect to the Nash Equilibrium of the \emph{calibrated} policy-level game decays at a rate of $O(1/\sqrt{T})$.
\end{corollary}
\begin{proof}
    See Appendix~\ref{app:proof2}.
\end{proof}

\begin{proposition}[$\epsilon$-Nash Equilibrium of the Regularized Game]
\label{prop:nash_bound}
Consider the regularized symmetric constant-sum game induced by S-SPPO, with the payoff function defined as:
\begin{equation}
    U_{\lambda}(\pi, \pi') = \mathbb{P}_{c}(\pi > \pi'|x) - \lambda \mathcal{R}_{rep}(\pi) + \lambda \mathcal{R}_{rep}(\pi')
\end{equation}
where $\mathbb{P}_{c}$ is the calibrated policy-level win rate, and $\lambda \mathcal{R}_{rep}$ denotes the geometric regularization term. Assuming the normalized latent embeddings bound the representation term such that $|\mathcal{R}_{rep}(\cdot)| \le M$, let $\pi_{\lambda}^{\ast}$ be the Nash Equilibrium of this regularized game. Then, $\pi_{\lambda}^{\ast}$ constitutes an $\epsilon$-Nash Equilibrium of the original calibrated game without latent repulsion, where the suboptimality gap is bounded by $\epsilon \le 2\lambda M$.
\end{proposition}

\begin{proof}
See Appendix~\ref{app:proof3}.
\end{proof}

\section{Experiments}

\subsection{Experimental Setup} \label{sec:exp-setup}
\input{tables/AlpacaEval}
\noindent\textbf{Models and Datasets.}
Following~\citet{wu2024self}, we conduct our experiments using two base models: {Mistral-7B-Instruct-v0.2} \citep{jiang2023mistral} and {Llama-3-8B-Instruct}. For the dataset, we utilize {UltraFeedback} \citep{cui2023ultrafeedback} exclusively as a prompt source, extracting approximately 60k prompts while \textit{discarding the accompanying responses}. Adopting the iterative strategy from SPPO \citep{wu2024self}, we split these prompts into three portions, using one portion per iteration to prevent overfitting. Following \citep{wu2024self}, we employ {PairRM} \citep{jiang2023llm}, a 0.4B parameter pairwise preference model, as our preference oracle to annotate generated responses and estimate win rates.

\noindent\textbf{Baselines.} We compare our method against a comprehensive set of baselines covering both base models and state-of-the-art alignment algorithms following~\citet{wu2024self}. First, we evaluate the instruction-tuned Base Models ({Mistral-7B-Instruct-v0.2} and {Llama-3-8B-Instruct}) directly. Second, we benchmark against {SPPO} \citep{wu2024self}, which serves as our primary baseline. Third, for the sake of completeness, we implement {Iterative DPO} \citep{rafailov2023direct} and {Iterative IPO} \citep{azar2024general} by aligning their response generation and training pipelines exactly with SPPO, differing solely in the substitution of the optimization objective (i.e., replacing the SPPO loss with DPO or IPO loss). We also maintain the same model selection scheme (optimizing based on PairRM win rates) to ensure consistency. Additionally, we include {Snorkel (Mistral-PairRM-DPO)}(see Appendix~\ref{app:exp-details-weight}), a widely-used model trained via three rounds of iterative DPO using PairRM. Finally, we include {Self-Rewarding LM} \citep{yuan2024self} as a representative of self-play methods that utilize the LLM itself as a judge.

\noindent\textbf{Evaluation Benchmarks.} We assess performance across four widely recognized benchmarks to evaluate diverse capabilities. To measure instruction-following proficiency in realistic user interaction scenarios, we report win rates and length-controlled win rates on {AlpacaEval 2.0} \citep{dubois2024length}. We evaluate multi-turn reasoning and conversation quality on the rigorous {MT-Bench} \citep{zheng2024judging}, and the judge directly assigns the score. For the SPPO baseline, we directly report the results from \citet{wu2024self} where available. For benchmarks not covered in the original paper, we conduct evaluations using their official released checkpoints. Furthermore, for challenging open-ended queries, we utilize {Arena-Hard} \citep{li2024crowdsourced} (v0.1, consistent with SPPO setup), a benchmark distinguished by its high separability and strong correlation with human rankings in Chatbot Arena. We report with win rates.  


\noindent\textbf{Implementation Details.}
During the generation phase, we sample $K=5$ responses per prompt strictly following the setup of SPPO~\citep{wu2024self}. The winner and loser pairs for optimization are selected based on PairRM scores. Due to space constraints, detailed hyperparameters, probability estimation formulas, and extensive baseline configurations are provided in Appendix~\ref{app:exp-details}.

\input{tables/AlpacaEval_leaderboard}
\subsection{Experimental Results}
\label{sec:exp_results}

We evaluate \textit{S}-SPPO across three standard benchmarks. In general, we observe that post-training alignment methods yield improvements over instruction-tuned base models. However, \textit{S}-SPPO consistently achieves the most substantial gains among all compared methods, effectively addressing the stability issues inherent in standard self-play.

\noindent\textbf{Instruction Following.}
As detailed in Table~\ref{tab:alpaca_main} and Table~\ref{tab:leaderboard}, while most iterative alignment baselines improve upon the base models, \textit{S}-SPPO secures the largest performance leap. On Mistral-7B, for instance, our method effectively doubles the length-controlled (LC) win rate of the base model (from 14.72\% to \textbf{30.84\%}), outperforming the best result achieved by Snorkel (26.39\%) and other baselines. When compared directly to standard SPPO, our method demonstrates consistent superiority across all iterations. On Llama-3-8B, \textit{S}-SPPO reaches a final LC win rate of \textbf{40.95\%}, surpassing SPPO's 38.77\%. This advantage is further validated on the AlpacaEval 2.0 Leaderboard, where Llama-3-8B-S-SPPO (41.0\%) not only exceeds the SPPO baseline but also outperforms much larger proprietary models (e.g., Claude 3 Opus and GPT-4 0314). 
Most notably, when evaluated with the \textbf{GPT-4 Turbo} annotator, Llama-3-8B-S-SPPO achieves remarkable performance, reaching \textbf{52.19\%} and \textbf{47.46\%} in raw and LC win rates, respectively, setting a new standard for open-source models of this scale among self-play alignment methods without access to ground-truth references from the training dataset.

\noindent\textbf{Multi-Turn Conversation.}
In multi-turn evaluations (Table~\ref{tab:mtbench}), while some baselines struggle to maintain the conversational capabilities of the base model, \textit{S}-SPPO demonstrates the highest ceiling. Specifically, on Llama-3-8B, \textit{S}-SPPO achieves the highest average score of \textbf{8.22}. A critical comparison reveals the limitations of the uncalibrated baseline: standard SPPO exhibits a downward trend on Llama-3 ($8.01 \rightarrow 7.93$) and an initial regression on Mistral-7B (falling below the base model to 7.21). 
In contrast, \textit{S}-SPPO effectively mitigates this degradation, yielding {consistent performance gains} across iterations. This suggests that our dual-space calibration successfully stabilizes the self-play optimization, preventing the rapid collapse often observed in the uncalibrated baseline.

\input{tables/MTBench}

\noindent\textbf{Complex Open-Ended Queries.}
Finally, on the challenging Arena-Hard-Auto v0.1 benchmark (Table~\ref{tab:arena-hard}), we observe that while alignment generally improves performance over the base model (e.g., Mistral-7B improves from 12.6 to 23.9 with \textit{S}-SPPO), \textit{S}-SPPO again achieves the highest empirical results. Most notably, our method proves significantly more stable than standard SPPO. On Llama-3, standard SPPO suffers from a severe performance collapse, dropping from 31.0 to 29.8 across iterations. \textit{S}-SPPO successfully avoids this overfitting, maintaining stable learning dynamics and outperforming the standard SPPO baseline on Mistral-7B (23.9 vs. 23.3), thereby confirming the necessity of dual-space calibration for handling high-complexity prompts.

\input{tables/Arena}

\subsection{Analysis}

\label{sec:ablation}


\paragraph{Impact of Dual-Space Calibration.}
To validate the individual and synergistic contributions of our proposed Dual-Space Semantic Calibration, we conduct an ablation study on Llama-3-8B by selectively removing the Supervision Calibration and Representation Calibration modules. As summarized in Table~\ref{tab:aba1}, the full \textit{S}-SPPO framework achieves the highest performance with a peak length-controlled (LC) win rate of \textbf{47.46\%}. In contrast, the uncalibrated baseline (Row 1) lags significantly behind at 44.79\%. Introducing components individually reveals that both are indispensable: applying Supervision Calibration alone (Row 2) improves performance to 46.16\% by filtering noisy preference labels, while Representation Calibration alone (Row 3) yields 45.20\% by preserving latent diversity between chosen and rejected samples. These results confirm that \textit{S}-SPPO's superior performance stems from the joint effect of calibrating optimization targets and regularizing the feature space, rather than any single mechanism in isolation.

\input{tables/aba-1}

\noindent\textbf{Additional Analysis.} We further investigate the sensitivity of the encoder $\phi(\cdot)$ in Appendix~\ref{app:aba-4}, computational efficiency in Appendix~\ref{app:efficicency}, and the sensitivity of $\lambda$ in Appendix~\ref{app:aba-3}.

\vspace{-0.05in}
\section{Related Work}

Reinforcement Learning from Human Feedback~\citep[RLHF,][]{christiano2017deep,ouyang2022training} has emerged as a promising paradigm for aligning LLMs with human preference. Common RLHF approaches, such as Proximal Policy Optimization~\citep[PPO,][]{schulman2017proximal}, typically rely on an explicit reward model, which can be difficult to design and reliably validate in many settings. DPO~\citep[][]{rafailov2023direct} mitigates this challenge by using LLMs as implicit rewards. We, henceforth, direct our main focus on this family of methods.

\textbf{Preference Optimization Objectives.}
Many efforts have been devoted to different preference optimization objectives beyond DPO. SLic-HF~\citep{zhao2023slic} introduces a margin-calibrated Slic loss that enforces a winner–loser separation. IPO~\cite{azar2024general} proposes a KL-regularized preference-optimization objective that applies an identity transform to fit pairwise preferences directly, without relying on the standard Bradley–Terry formulation. KTO~\cite{ethayarajh2024kto} derives a preference-optimization loss from Kahneman–Tversky utility, increasing scores for “kept” responses and decreasing them for “rejected” ones. Odds Ratio Preference Optimization~\citep[ORPO,][]{hong2024orpo} enables supervised fine-tuning and preference alignment in a single training run, without requiring an intermediate reference policy. Simple Preference Optimization~\citep[SimPO,][]{meng2024simpo} introduces a reference-free objective that mitigates length bias by maximizing the gap between the preferred and dispreferred responses’ average log-likelihood. 

\textbf{Offline and Online Preference Optimization.}
Most prior DPO variants operate in an offline setting~\cite{rafailov2023direct, ethayarajh2024kto,hong2024orpo,meng2024simpo}, optimizing solely on a fixed and pre-collected preference dataset. However, without an explicit reward model, they cannot sample preference pairs from the optimal policy, which hinders their ability to capture the intransitivity and irrationality in human preferences. To address this, early work augments the preference dataset by sampling from a trained SFT policy~\cite{zhao2023slic} or from a further-refined SFT policy via rejection sampling~\cite{liu2024statistical}, enabling learning from data that more closely reflects the optimal policy’s behavior. 
Recently, many studies have extended this idea to an iterative, online setting, either by continually updating the reference model to the latest policy or by generating fresh preference pairs at each training iteration~\cite{xiong2024iterative,dong2024rlhf,rosset2024direct,kim2025sdpo,wu2024self}. 
Concurrently, various works have explored alternative game-theoretic formulations~\cite{wang2025magnetic,pasztor2025stackelberg,tang2025game} or autonomous self-improvement mechanisms~\cite{ji2024self,chen2024self,le2025token,xiang2025self} to enhance stability.
In this paper, we focus specifically on Self-Play Preference Optimization~\citep[SPPO,][]{wu2024self} because it is one of the most fundamental works for addressing intransitivity and irrationality by directly optimizing preference probabilities within a constant-sum game framework, rather than relying on reward maximization. Moreover, it generates on-policy self-play training signals using a lightweight preference model, without requiring human-annotated preference data.

\vspace{-0.05in}
\section{Conclusion}
In this work, we identify a critical instability in self-play preference optimization: policy degeneration occurs when preference oracles assign overconfident wins to semantically indistinguishable responses. To address this, we propose \textit{S}-SPPO, a dual-space calibration framework that stabilizes alignment via Supervision Calibration (semantic gating) and Representation Calibration (latent repulsion). Theoretically, we prove that \textit{S}-SPPO preserves the underlying constant-sum game structure, theoretically supporting convergence to a Nash Equilibrium. Empirically, \textit{S}-SPPO sets a new standard for open-source models, achieving a 52.19\% win rate and 47.46\% length-controlled win rate on AlpacaEval 2.0 without relying on human-annotated data. We provide further discussion on limitations and future directions in Appendix~\ref{app:discuss}.



\section*{Impact Statement}

This paper addresses the instability in Large Language Model alignment. We see three main potential impacts of our work:

\textbf{Improved Reliability.} Standard SPPO paradigm often forces models to fit noise in the reward model. By neutralizing supervision signals for semantically indistinguishable pairs, our method discourages the policy from overfitting to spurious preferences. This leads to models that are less likely to hallucinate differences in ambiguous situations.

\textbf{Preserved Diversity.} Alignment training frequently causes \textit{mode collapse}, where models become repetitive. Our method forces the model to maintain distinct internal representations. This ensures the model retains its variety and does not degenerate into a narrow set of outputs.

\textbf{Reduced Reliance on Human Data.} Self-play optimization generates training signals automatically using a reward model. This removes the need for expensive human-annotated preference labels. By making this process stable and effective, our framework allows models to improve autonomously without the high cost of manual data collection.

\bibliography{example_paper}
\bibliographystyle{icml2026}

\newpage
\appendix

\input{appendix}

\end{document}

%% file: tables/AlpacaEval.tex
\begin{table*}[!t]
    \centering
    \caption{\textbf{AlpacaEval 2.0 evaluation against GPT-4 Turbo.} We report results using both the GPT-4 Annotator and the GPT-4 Turbo Annotator (see Appendix~\ref{app:exp-details-benchmark}) as judges. For the SPPO baseline, we directly report the results from \citet{wu2024self} where available (i.e., GPT-4 Annotator); otherwise, we evaluate using the official released checkpoints (i.e., GPT-4 Turbo Annotator, see Appendix~\ref{app:exp-details-weight}). Across both evaluation protocols, \textit{S}-SPPO consistently outperforms baselines in both raw and length-controlled win rates.}
    \resizebox{0.8\textwidth}{!}{%
    \begin{tabular}{l | c c c | c c c }
    \toprule
    \multirow{2}{*}{Model} & \multicolumn{3}{c|}{GPT-4 Annotator} & \multicolumn{3}{c}{GPT-4 Turbo Annotator} \\
     & LC Win & Win Rate & Avg. Len & LC Win & Win Rate & Avg. Len \\
    \midrule
    Mistral-7B-Instruct-v0.2  & 17.11 & 14.72 & 1676 & 18.84 & 15.38 & 1622 \\
    \midrule
    Snorkel (Mistral-PairRM-DPO) & 26.39 & 30.22 & 2736 & - & - & - \\
    \midrule 
    Self-Rewarding 70B Iter1 & - & 9.94 & 1092 & - & - & - \\
    Self-Rewarding 70B Iter2 & - & 15.38 & 1552 & - & - & - \\
    Self-Rewarding 70B Iter3 & - & 20.44 & 2552 & - & - & - \\
    \midrule
    \midrule
    Mistral-7B-DPO Iter1   & 23.81 & 20.44 & 1723 & - & - & - \\
    Mistral-7B-DPO Iter2   & 24.23 & 24.46 & 2028 & - & - & - \\
    Mistral-7B-DPO Iter3  & 22.30 & 23.39 & 2189 & - & - & - \\
    \midrule
    Mistral-7B-IPO Iter1   & 23.78 & 20.77 & 1693 & - & - & - \\
    Mistral-7B-IPO Iter2   & 21.08 & 23.38 & 2660 & - & - & - \\
    Mistral-7B-IPO Iter3  & 20.06 & 22.47 & 2760 & - & - & - \\
    \midrule
    Mistral-7B-SPPO Iter1  & $24.79$ & $23.51$ & 1855 & 28.80 & 29.51 & 2028 \\
    Mistral-7B-SPPO Iter2  & $26.89$ & $27.62$ & 2019 & 32.60 & 35.31 & 2172 \\
    Mistral-7B-SPPO Iter3  & $28.53$ & $31.02$ & 2163 & 34.90 & 39.46 & 2294 \\
    \midrule
    \rowcolor[HTML]{EFEFEF}Mistral-7B-\textit{S}-SPPO Iter1  & 24.28 & 25.02 & 2039 & 28.34 & 29.18 & 2039 \\
    \rowcolor[HTML]{EFEFEF}Mistral-7B-\textit{S}-SPPO Iter2  & 28.22 & 31.42 & 2206 & 34.89 & 38.04 & 2206 \\
    \rowcolor[HTML]{EFEFEF}Mistral-7B-\textit{S}-SPPO Iter3  & 30.84 & 35.09 & 2326 & 36.13 & 40.26 & 2326 \\
    \midrule
    \midrule
    Llama-3-8B-Instruct  & 22.92 & 22.57 &  1899 & 28.31 & 28.52 & 1975 \\
    \midrule
    Llama-3-8B-SPPO Iter1  & $31.73$ & $31.74$ &  1962 & 35.69 & 36.29 & 2010 \\
    Llama-3-8B-SPPO Iter2  & $35.15$ & $35.98$ &  2021 & 42.24 & 43.42 & 2070 \\
    Llama-3-8B-SPPO Iter3  & $38.77$ & $39.85$ &   2066 & 44.79 & 46.84 & 2128 \\
    \midrule
    \rowcolor[HTML]{EFEFEF}Llama-3-8B-\textit{S}-SPPO Iter1 & 31.33 & 34.72 & 2114 & 38.18 & 40.37 & 2114 \\
\rowcolor[HTML]{EFEFEF}Llama-3-8B-\textit{S}-SPPO Iter2 & 37.07 & 42.07 & 2208 & 44.77 & 48.47 & 2208 \\
\rowcolor[HTML]{EFEFEF}Llama-3-8B-\textit{S}-SPPO Iter3 &\textbf{40.95} & \textbf{46.52} & 2287 & \textbf{47.46} & \textbf{52.19} & 2287 \\
    \bottomrule
    \end{tabular}
    }
    \label{tab:alpaca_main}
\end{table*}

%% file: tables/AlpacaEval_leaderboard.tex
\begin{table}[t!]
    \centering
    \caption{\textbf{Comparisons on the AlpacaEval 2.0 Leaderboard.} We compare our method against proprietary models (e.g., GPT-4 Turbo, Claude 3 Opus) and other top-tier open-source models.}
    \resizebox{0.9\columnwidth}{!}{%
\begin{tabular}{l | c c  }
\toprule
\multirow{2}{*}{Model} & \multicolumn{2}{c}{AlpacaEval 2.0}   \\
& LC Win Rate  & Win Rate  \\
\midrule
GPT-4 Turbo & 50.0 & 50.0 \\
\rowcolor[HTML]{EFEFEF} Llama-3-8B-\textit{S}-SPPO Iter3 & 41.0 & 46.5 \\
 Llama-3-8B-SPPO Iter3 & 38.8 & 39.9 \\
Claude 3 Opus & 40.5 & 29.1  \\
GPT-4 0314  & 35.3 & 22.1  \\
Llama 3 70B Instruct & 34.4 & 33.2 \\
\rowcolor[HTML]{EFEFEF} Mistral-7B-\textit{S}-SPPO Iter3  & 30.8 & 35.1    \\
 Mistral-7B-SPPO Iter3  & 28.5 & 31.0    \\
GPT-4 0613 & 30.2 & 15.8   \\
Mistral Medium  & 28.6 & 21.9   \\
Claude 2 & 28.2 & 17.2 \\
Snorkel (Mistral-PairRM-DPO) & 26.4 & 30.2  \\
Gemini Pro & 24.4 & 18.2 \\
Mistral 8$\times$7B v0.1 & 23.7 & 18.1 \\
Llama 3 8B Instruct & 22.9 & 22.6 \\
\bottomrule
\end{tabular}
    }
    \label{tab:leaderboard}
\end{table}

%% file: tables/MTBench.tex
\begin{table}[t!]
    \centering
    \caption{\textbf{MT-Bench evaluation for multi-turn conversation capabilities.} We focus on the average score. Unlike SPPO, which shows a downward trend across iterations on Llama-3 (8.01 $\to$ 7.93), \textit{S}-SPPO demonstrates stable or increasing performance trends. Notably, \textit{S}-SPPO achieves the highest average score, surpassing even the best-performing iteration of all baselines.}
    \label{tab:mt-bench}
    \resizebox{0.90\columnwidth}{!}{%
    \begin{tabular}{l | c c | c}
    \toprule
    \multirow{2}{*}{Model} & \multicolumn{3}{c}{MT-Bench}   \\
    & 1st Turn & 2nd Turn & Average \\
    \midrule
    Mistral-7B-Instruct-v0.2  & 7.78 & 7.25& 7.51  \\
    Snorkel (Mistral-PairRM-DPO) & 7.83 & 7.33 & 7.58  \\
    \midrule
    Mistral-7B-DPO Iter1  & 7.45 & 6.58 & 7.02 \\
    Mistral-7B-DPO Iter2  & 7.57 & 6.56 & 7.06 \\
    Mistral-7B-DPO Iter3  & 7.49 & 6.69 & 7.09 \\
    \midrule
    Mistral-7B-SPPO Iter1  & 7.63 & 6.79 & 7.21   \\
    Mistral-7B-SPPO Iter2  & 7.90 & 7.08 & 7.49   \\
    Mistral-7B-SPPO Iter3  & 7.84 & 7.34 & 7.59  \\
    \midrule
    \rowcolor[HTML]{EFEFEF}Mistral-7B-\textit{S}-SPPO Iter1  & 7.84 & 7.26 & 7.55   \\
    \rowcolor[HTML]{EFEFEF}Mistral-7B-\textit{S}-SPPO Iter2  & 8.04 & 7.05 & 7.55   \\
    \rowcolor[HTML]{EFEFEF}Mistral-7B-\textit{S}-SPPO Iter3  & 7.96 & 7.36 & 7.66  \\
    \midrule
    Llama-3-8B-SPPO Iter1  & 8.29 & 7.73 & 8.01   \\
    Llama-3-8B-SPPO Iter2  & 8.33 & 7.61 & 7.97   \\
    Llama-3-8B-SPPO Iter3  & 8.36 & 7.49 & 7.93  \\
    \midrule
    \rowcolor[HTML]{EFEFEF}Llama-3-8B-\textit{S}-SPPO Iter1  & 8.20 & 7.85 & 8.03   \\
    \rowcolor[HTML]{EFEFEF}Llama-3-8B-\textit{S}-SPPO Iter2  & 8.60 & 7.82 & 8.21   \\
    \rowcolor[HTML]{EFEFEF}Llama-3-8B-\textit{S}-SPPO Iter3  & 8.54 & 7.90 & \textbf{8.22} \\
    \bottomrule
    \end{tabular}
    }\label{tab:mtbench}
\end{table}

%% file: tables/Arena.tex
\begin{table}[t!]
    \centering
    \caption{\textbf{Arena-Hard-Auto evaluation.} We observe that standard SPPO suffers from performance degradation on Llama-3 across iterations. In contrast, \textit{S}-SPPO successfully mitigates this collapse, demonstrating consistent performance improvements.}
    \label{tab:arena-hard}
    \resizebox{0.8\columnwidth}{!}{%
    \begin{tabular}{l | c}
    \toprule
    Model & Arena-Hard-Auto   \\
    \midrule
    Mistral-7B-Instruct  & 12.6  \\
    \midrule
    Snorkel (Mistral-PairRM-DPO) & 20.7  \\
    \midrule
    Mistral-7B-SPPO Iter1  & 18.7   \\
    Mistral-7B-SPPO Iter2  & 20.4   \\
    Mistral-7B-SPPO Iter3  & 23.3  \\
    \midrule
    \rowcolor[HTML]{EFEFEF}Mistral-7B-\textit{S}-SPPO Iter1  & 21.5   \\
    \rowcolor[HTML]{EFEFEF}Mistral-7B-\textit{S}-SPPO Iter2  & 21.8 \\
    \rowcolor[HTML]{EFEFEF}Mistral-7B-\textit{S}-SPPO Iter3  & 23.9  \\
    \midrule
    Llama-3-8B-SPPO Iter1  & 31.0   \\
    Llama-3-8B-SPPO Iter2  & 30.1   \\
    Llama-3-8B-SPPO Iter3   & 29.8  \\
    \midrule
   \rowcolor[HTML]{EFEFEF} Llama-3-8B-\textit{S}-SPPO Iter1  & 30.0   \\
   \rowcolor[HTML]{EFEFEF} Llama-3-8B-\textit{S}-SPPO Iter2  & 30.6   \\
   \rowcolor[HTML]{EFEFEF} Llama-3-8B-\textit{S}-SPPO Iter3  & 31.5  \\
    \bottomrule
    \end{tabular}
    }
\end{table}

%% file: tables/aba-1.tex

\begin{table}[!t]
\centering
\caption{Ablation study on Dual-Space Semantic Calibration. We analyze the impact of removing Supervision Calibration and Representation Calibration. The results demonstrate that both components are essential, with the full S-SPPO framework achieving the highest win rates.}
\label{tab:aba1}
\resizebox{0.40\textwidth}{!}{%
\begin{tabular}{cc|ccc}\toprule
\multirow{2}{*}{\begin{tabular}[c]{@{}c@{}}Supervision \\ Calibration\end{tabular}} & \multirow{2}{*}{\begin{tabular}[c]{@{}c@{}}Representation \\ Calibration\end{tabular}} & \multicolumn{3}{c}{LC Win Rate}   \\ 
                                                                                    &                                                                                        & Iter. 1              & Iter. 2              & Iter. 3              \\ \midrule
\ding{55}                                                                                   & \ding{55}                                                                                      & 35.69                & 42.24                & 44.79                \\
\ding{51}                                                                                   & \ding{55}                                                                                      & 38.39 & 43.91 & 46.16 \\
\ding{55}                                                                                   & \ding{51}                                                                                       & 37.01 & 43.79 & 45.20 \\
\ding{51}                                                                                   & \ding{51}                                                                                       & 38.18                & 44.77                & 47.46          \\ \bottomrule     
\end{tabular}%
}
\vspace{-0.05in}
\end{table}
\vspace{-0.07in}

%% file: appendix.tex
\onecolumn


\section{Proof}\label{app:proof}

\subsection{Proof of Lemma~\ref{lem:policy_constant_sum}}\label{app:proof1}

\begin{proof}
By definition, the policy-level win rate is the expectation of the calibrated pairwise preference. Thus, it suffices to show that the calibrated pairwise probability $\hat{\PP}_c$ maintains the antisymmetry property of the original preference oracle.

For any response pair $(\yb_i, \yb_j)$, let $S_{ij} = S_{ji}$ denote their semantic similarity. The calibrated pairwise probability is given by:
\begin{equation}
\label{eqn:pairwise_proof}
\hat{\PP}_c(\yb_i \succ \yb_j) = \frac{1}{2} + (1 - S_{ij}) \left( \hat{\PP}(\yb_i \succ \yb_j) - \frac{1}{2} \right).
\end{equation}
We examine the reverse pair $(\yb_j, \yb_i)$. Since the original preference oracle is a valid probability distribution, it satisfies $\hat{\PP}(\yb_j \succ \yb_i) = 1 - \hat{\PP}(\yb_i \succ \yb_j)$. Substituting this into the calibration formula:
\begin{align*}
\hat{\PP}_c(\yb_j \succ \yb_i) &= \frac{1}{2} + (1 - S_{ji}) \left( \hat{\PP}(\yb_j \succ \yb_i) - \frac{1}{2} \right) \\
&= \frac{1}{2} + (1 - S_{ij}) \left( \left(1 - \hat{\PP}(\yb_i \succ \yb_j)\right) - \frac{1}{2} \right) \\
&= \frac{1}{2} + (1 - S_{ij}) \left( \frac{1}{2} - \hat{\PP}(\yb_i \succ \yb_j) \right) \\
&= \frac{1}{2} - (1 - S_{ij}) \left( \hat{\PP}(\yb_i \succ \yb_j) - \frac{1}{2} \right) \\
&= 1 - \left[ \frac{1}{2} + (1 - S_{ij}) \left( \hat{\PP}(\yb_i \succ \yb_j) - \frac{1}{2} \right) \right] \\
&= 1 - \hat{\PP}_c(\yb_i \succ \yb_j).
\end{align*}
Since $\hat{\PP}_c(\yb_i \succ \yb_j) + \hat{\PP}_c(\yb_j \succ \yb_i) = 1$ holds for any realization $(\yb_i, \yb_j)$, taking the expectation over $\yb \sim \pi(\cdot|\xb)$ and $\yb' \sim \pi'(\cdot|\xb)$ preserves this equality by linearity. Thus, the policy-level game remains constant-sum.
\end{proof}

\subsection{Proof of Corollary~\ref{cor:convergence}}\label{app:proof2}
\begin{proof}
Lemma~\ref{lem:policy_constant_sum} establishes that the semantic calibration preserves the symmetric constant-sum structure. Accordingly, the \textit{S}-SPPO update can be interpreted as approximating multiplicative weight updates (MWU) on this calibrated game. Standard no-regret learning results for constant-sum games \citep{freund1999adaptive,wu2024self} apply directly, facilitating convergence to the Nash Equilibrium of the calibrated game defined by $\PP_c$.
\end{proof}

\subsection{Proof of Proposition~\ref{prop:nash_bound}}\label{app:proof3}
\begin{proof}
The asymmetric construction of the regularization terms ensures that $U_{\lambda}(\pi, \pi') + U_{\lambda}(\pi', \pi) = 1$, preserving the strictly constant-sum property. By the definition of the Nash Equilibrium for $U_{\lambda}$, it holds that $U_{\lambda}(\pi, \pi_{\lambda}^{\ast}) \le 1/2$ for any policy $\pi$. Substituting the payoff definition yields:
\begin{equation}
    \mathbb{P}_{c}(\pi > \pi_{\lambda}^{\ast}|x) - \lambda \mathcal{R}_{rep}(\pi) + \lambda \mathcal{R}_{rep}(\pi_{\lambda}^{\ast}) \le \frac{1}{2}
\end{equation}
Rearranging the terms, we obtain:
\begin{equation}
    \mathbb{P}_{c}(\pi > \pi_{\lambda}^{\ast}|x) \le \frac{1}{2} + \lambda \big( \mathcal{R}_{rep}(\pi) - \mathcal{R}_{rep}(\pi_{\lambda}^{\ast}) \big)
\end{equation}
Given the boundedness assumption $|\mathcal{R}_{rep}(\cdot)| \le M$, the maximum deviation from the original equilibrium ($1/2$) is strictly bounded by:
\begin{equation}
    \mathbb{P}_{c}(\pi > \pi_{\lambda}^{\ast}|x) \le \frac{1}{2} + \lambda \big( M - (-M) \big) = \frac{1}{2} + 2\lambda M
\end{equation}
This concludes the proof, demonstrating that the theoretical deviation from the unregularized Nash Equilibrium is linearly bounded by the regularization weight $\lambda$.
\end{proof}

\begin{algorithm}[t!]
\caption{Algorithmic Summary of proposed \textit{S}-SPPO.}
\label{alg:s_sppo}
\begin{algorithmic}[1]
    \STATE \textbf{input}:
    base policy $\pi_{\btheta_1}$,
    preference oracle $\PP$,
    semantic encoder $\phi(\cdot)$,
    scaling factor $\frac{1}{\beta}$, 
    regularization weight $\lambda$,
    number of generated samples $K$.
    
    \FOR{$t=1,2,\dots$}
        \STATE Generate responses $\yb_{1:K} \sim \pi_{t}(\cdot|\xb)$ for prompt $\xb \sim \cX$. \label{line:generate}
        
        \STATE Annotate win rates $\hat{P}(\yb \succ \pi_t | \xb)$ for all $\yb \in \yb_{1:K}$. \label{line:annotate}
        
        \STATE \textbf{Selection}: Identify winner $\yb_w$ and loser $\yb_l$ based on scores. \label{line:select}
        
        \STATE \myhighlight{\textbf{Supervision Calibration}: 
        Compute similarity $S = \text{sim}(\phi(\yb_w), \phi(\yb_l))$ and calibrate targets for both $y \in \{\yb_w, \yb_l\}$:
        \begin{equation*}
            \hat{P}_c(\yb) \leftarrow \frac{1}{2} + \left(1 - S\right)\left(\hat{P}(\yb) - \frac{1}{2}\right)
        \end{equation*}
        \textit{(Symmetrically pulls both winner and loser targets towards 0.5)}} \label{line:calibrate}
        
        \STATE Form dataset $\cD_t = \{(\xb, \yb, \hat{P}_c(\yb)) \mid \yb \in \{\yb_w, \yb_l\} \}$.
        
        \STATE Optimize $\pi_{\btheta_{t+1}}$ with \textbf{Representation Calibration}: \label{line:optimize}
        
        \myhighlight{
        \vspace{-0.4cm} 
        \begin{align}
            \btheta_{t+1}
            &\leftarrow 
            \argmin_{\btheta}
            \EE_{(\xb, \yb, \hat{P}_c) \sim \cD_t}
            \Bigg[
            \underbrace{\bigg(
            \log \frac{\pi_{\btheta}(\yb|\xb)}{\pi_{t}(\yb|\xb)}
            -
            \frac{1}{\beta} \bigg(\hat{P}_c - \frac{1}{2}
            \bigg)
            \bigg)^2}_{\text{Regression on Calibrated Targets}} \nonumber+ \lambda \cdot h_{\btheta}(\yb_w)^\top h_{\btheta}(\yb_l) \Bigg]
            \label{eqn:opt_sppo}
        \end{align}
        }
    \ENDFOR
\end{algorithmic}
\end{algorithm}

\section{Detailed Experimental Settings} \label{app:exp-details}

\subsection{Training Hyperparameters}
All experiments are conducted on 8 $\times$ Nvidia H100 GPUs. For our proposed method ({S}-SPPO), we train for a total of three iterations. In each iteration, we select the checkpoint with the highest average win rate (judged by PairRM-0.4B on a hold-out subset of UltraFeedback) to proceed to the next round.
We trained the model for 18 epochs using a global batch size of 64 and we set $\beta = 1\text{e}-3$, with a learning rate of $5.0\text{e}-7$ that follows a linear decay schedule with a warmup ratio of 0.1.
These settings were selected to ensure stability and fair comparison across all iterative baselines ({S}-SPPO, Iterative DPO, Iterative IPO). $\lambda$ is set to 1. We choose \texttt{all-MiniLM-L6-v2} as our default lightweight encoder model.

\subsection{Response Generation and Probability Estimation}
During the generation phase of each iteration, we sample $K=5$ responses for each prompt from the current policy using a temperature of $1.0$ and top-$p=1.0$.
To estimate the winning probability distribution $\hat{P}(\yb_i \succ \pi_t |\xb)$, we calculate the average win rate of a response $\yb_i$ against all other sampled responses $\yb_k$:
\begin{align}
    \hat{P}(\yb_i \succ \pi_t |\xb) = \frac{1}{K} \sum_{k=1}^{K} \PP(\yb_i \succ \yb_k | \xb), \quad \forall i \in [K].
\end{align}
The pairwise probability $\PP(\yb \succ \yb' | \xb)$ is derived from the PairRM scoring function $s(\yb, \yb' ; \xb)$, which reflects the relative strength difference:
\begin{align}
    \PP(\yb \succ \yb' | \xb) = \frac{\exp(s(\yb, \yb' ; \xb))}{1 + \exp(s(\yb, \yb' ; \xb))}.
\end{align}
Unlike Bradley-Terry models which assume transitivity, PairRM assigns relative rewards directly, allowing it to model general, potentially non-transitive preferences.
For preference pair construction, following Snorkel (see next section for the link), we select the response with the highest PairRM score as the winner $\yb_w$ and the lowest as the loser $\yb_l$, where the score is defined as:
\begin{align}
    s_{\text{PairRM}}(\yb_i ; \xb) := \frac{1}{K} \sum_{k=1}^{K} s(\yb_i, \yb_k; \xb).
\end{align}

\subsection{Baseline Model and Encoder Details}\label{app:exp-details-weight}
\begin{itemize}[leftmargin=*]
\item{Mistral-7B-Instruct-v0.2}: {\url{https://huggingface.co/mistralai/Mistral-7B-Instruct-v0.2}}.
\item  Llama-3-8B-Instruct: {\url{https://huggingface.co/meta-llama/Meta-Llama-3-8B-Instruct}}.

\item  Snorkel (Mistral-PairRM-DPO): \url{https://huggingface.co/snorkelai/Snorkel-Mistral-PairRM-DPO}.

\item PairRM: \url{https://huggingface.co/llm-blender/PairRM}.

\item SPPO official checkpoints: \url{https://huggingface.co/collections/UCLA-AGI/sppo}.

\item all-MiniLM-L6-v2: \url{https://huggingface.co/sentence-transformers/all-MiniLM-L6-v2}. It has a model size of 22.7M.

\item jina-embeddings-v3~\cite{sturua2024jina}: \url{https://huggingface.co/jinaai/jina-embeddings-v3}. It has a model size of 0.6B.

\item nomic-ai/nomic-embed-text-v1.5~\cite{nussbaum2024nomic}: \url{https://huggingface.co/nomic-ai/nomic-embed-text-v1.5}. It has a model size of 0.1B.

\end{itemize}
\subsection{Benchmark Details}\label{app:exp-details-benchmark}
\begin{itemize}[leftmargin=*]
    \item \textbf{AlpacaEval 2.0:} Uses AlpacaFarm prompts \citep{dubois2024alpacafarm}. Responses are judged against a GPT-4-Turbo baseline using different annotators, i.e, judges. We report the length-controlled win rate. We provided their link here:
    \begin{itemize}
        \item GPT-4 annotator: \url{https://github.com/tatsu-lab/alpaca_eval/blob/main/src/alpaca_eval/evaluators_configs/weighted_alpaca_eval_gpt4_turbo/configs.yaml}.
        \item GPT-4 Turbo annotator: \url{https://github.com/tatsu-lab/alpaca_eval/blob/main/src/alpaca_eval/evaluators_configs/weighted_alpaca_eval_gpt4_turbo_new/configs.yaml}.
    \end{itemize}
    \item \textbf{Arena-Hard:} A dataset designed to have high separability and correlation with Chatbot Arena Elo ratings. It focuses on complex, open-ended queries \citep{li2024crowdsourced}.
    \item \textbf{MT-Bench:} Consists of 80 high-quality multi-turn questions covering coding, reasoning, and role-playing. Responses are graded on a scale of 1-10 by GPT-4 \citep{zheng2024judging}.
\end{itemize}

\section{More Analysis}

\subsection{Sensitivity of Encoder $\phi(\cdot)$}\label{app:aba-4}


In Fig.~\ref{fig:overview}, we use the encoder \texttt{all-MiniLM-L6-v2} to reflect the uncalibrated nature of the original SPPO paper~\cite{wu2024self}. Here, we demonstrate that this observation is consistent across different encoder architectures. To illustrate this, we plot the features extracted by representative encoders in Fig.~\ref{fig:more_encoders}, which reveals the same pattern of misalignment: the preference oracle assigns high confidence even when semantic similarity is high. Furthermore, as shown in Table~\ref{tab:encoder_ablation}, we observe that while the choice of encoder leads to minor fluctuations in final performance, our proposed \textit{S}-SPPO consistently outperforms the baseline across all iterations.

\begin{figure}[h]
    \centering
    \includegraphics[width=0.75\linewidth]{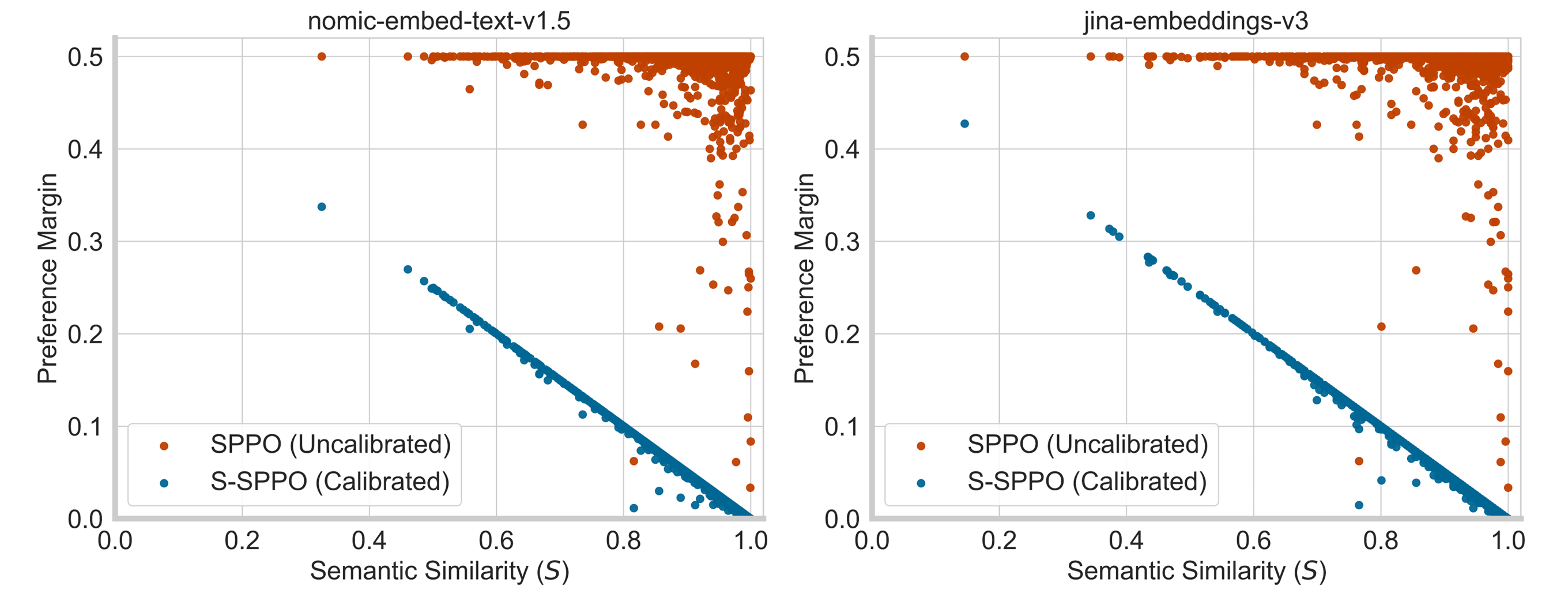}
    \caption{Uncalibration across different encoders. We visualize the preference margin against semantic similarity for \texttt{nomic-embed-text-v1.5} and \texttt{jina-embeddings-v3}. Similar to the observations with \texttt{all-MiniLM-L6-v2}, the uncalibrated baseline (SPPO) fails to converge to a neutral preference margin even for semantically indistinguishable pairs ($S \approx 1.0$), whereas \textit{S}-SPPO successfully calibrates the margin.}
    \label{fig:more_encoders}
\end{figure}

\begin{table}[h]
\centering
\caption{Ablation study on encoder selection. We report the Length-Controlled (LC) win rates on AlpacaEval 2.0. \textit{S}-SPPO achieves consistent gains over the baseline regardless of the specific semantic encoder used for calibration.}
\label{tab:encoder_ablation}
\resizebox{0.4\columnwidth}{!}{%
\begin{tabular}{lrrr}\toprule
                      & \multicolumn{1}{l}{Iter. 1} & \multicolumn{1}{l}{Iter. 2} & \multicolumn{1}{l}{Iter. 3} \\ \midrule
Baseline              & 35.7                        & 42.2                        & 44.8                        \\ \midrule
jina-embeddings-v3    & 36.8                        & 44.1                        & 47.7                        \\
nomic-embed-text-v1.5 & 39.0                          & 43.9                        & 46.9                        \\
all-MiniLM-L6-v2      & 38.2                        & 44.8                        & 47.5                        \\ \bottomrule    
\end{tabular}%
}
\end{table}

\subsection{Computational Efficiency Analysis}\label{app:efficicency}
We assess the computational overhead of \textit{S}-SPPO compared to the SPPO baseline. 
The implementation pipeline follows the procedure outlined in Algorithm~\ref{alg:s_sppo}, comprising three main phases: \textit{Generation \& Annotation} (Steps 3--4), \textit{Preference Construction} (Steps 5--7), and \textit{Policy Optimization} (Step 8).
As shown in Table~\ref{tab:efficiency}, the overhead introduced by our dual-space calibration is minor.
During the \textit{Preference Construction} phase, applying Supervision Calibration with a lightweight encoder adds only $\sim$5 seconds to the runtime ($90\text{s} \to 95\text{s}$) with a minimal GPU memory footprint of $\sim$1.7 GB.
In the \textit{Policy Optimization} phase, Representation Calibration requires accessing the model's hidden states, resulting in a \textit{slight increase} in peak GPU memory ($67.0\text{GB} \to 72.8\text{GB}$) due to computation graph storage, but virtually no increase in training time ($\sim$16 mins).
Crucially, the \textit{Total Pipeline} runtime remains approximately 66 minutes, as it is dominated by the generation and annotation phases (Steps 3--4) which are shared between methods. This confirms that \textit{S}-SPPO achieves significant performance gains with minimal impact on training efficiency.

\input{tables/aba-2}



\subsection{Sensitivity of Hyperparameter $\lambda$}\label{app:aba-3}

\begin{figure}[h]
    \centering
    \includegraphics[width=0.5\linewidth]{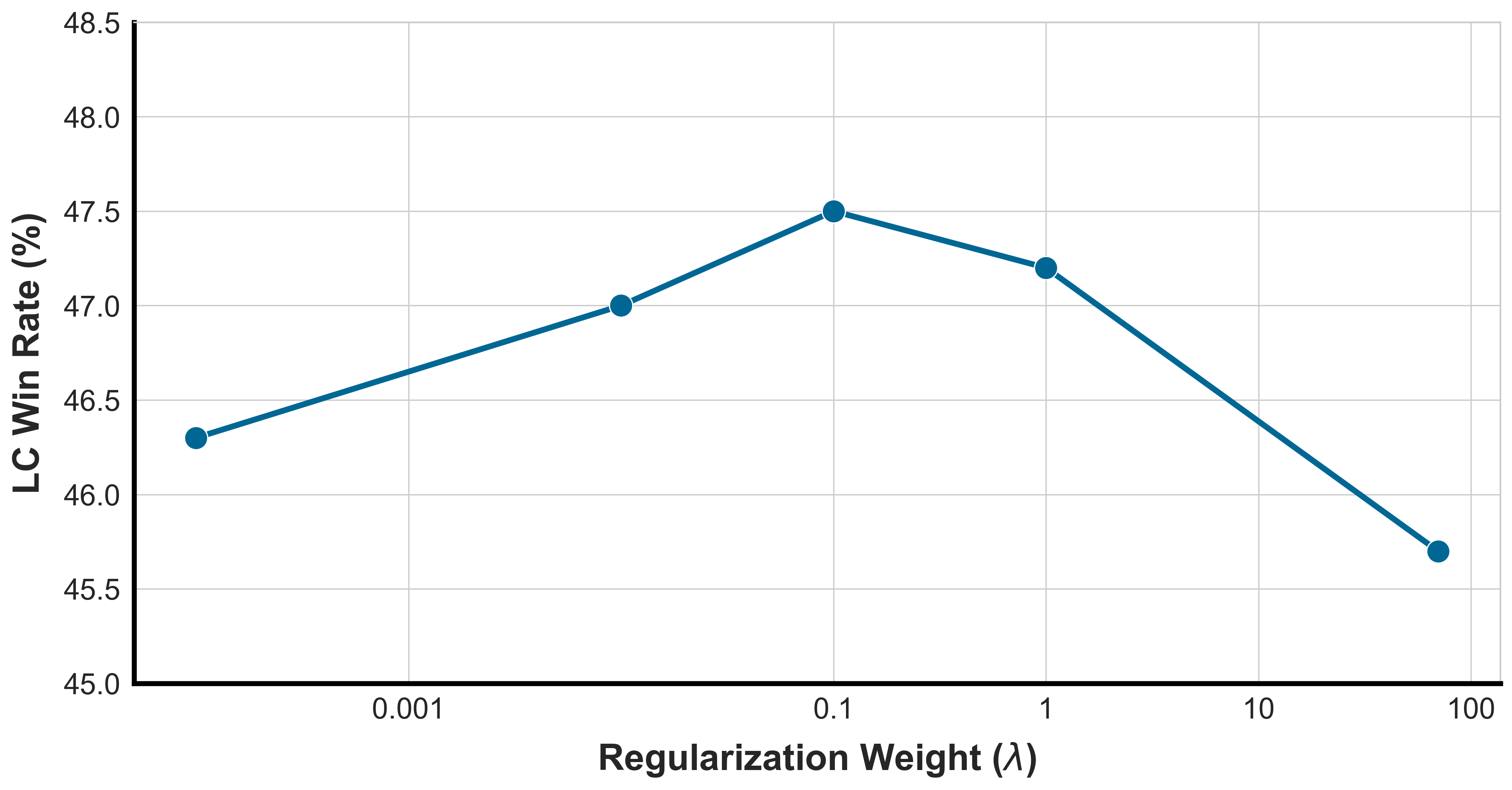}
    \caption{Sensitivity analysis of $\lambda$ on AlpacaEval 2.0.}
    \label{fig:lambda_sensitivity}
\end{figure}

In this section, we analyze the sensitivity of \textit{S}-SPPO to the regularization weight $\lambda$, which controls the strength of Representation Calibration. As illustrated in Fig.~\ref{fig:lambda_sensitivity}, we observe an inverted-U trend where performance peaks when $\lambda$ is in the moderate range (around $0.1$ to $1$). When $\lambda$ is too small, the repulsive force is insufficient, limiting the model's discriminative power. Conversely, performance degrades as $\lambda$ becomes excessively large. We attribute this to the disruption of the natural semantic manifold: since both chosen and rejected responses are generated by a capable policy, they naturally share semantic proximity. Forcing these responses to be strictly orthogonal via a large $\lambda$ imposes an unnatural geometric constraint that hinders learning. Consequently, a moderate $\lambda$ strikes the optimal balance, enforcing sufficient latent diversity without shattering the semantic structure of the policy.

\section{Discussion}\label{app:discuss}

\subsection{Theoretical Justification via Flat Reward Landscapes}
Recent theoretical studies~\cite{lee2025flat} suggest that a robust reward function should exhibit a flat reward landscape in the policy parameter space to ensure robustness against action or reward perturbations. In the context of Large Language Model (LLM) preference optimization, when two generated responses are semantically near-identical (i.e., with extremely high semantic similarity $S$), they essentially represent minor perturbations in the action space. Consequently, an ideal and well-calibrated preference oracle should maintain a flat evaluation landscape across these local variations, naturally outputting a neutral preference probability close to 0.5.In practice, however, uncalibrated reward models severely violate this property, frequently assigning extreme win rates to semantically indistinguishable response pairs. This miscalibration injects sharp, high-magnitude reward gradients onto a landscape that should theoretically be flat, thereby destabilizing the self-play optimization trajectory. \textit{S}-SPPO directly rectifies this fundamental flaw by progressively annealing the preference signal toward the maximum-entropy baseline (0.5) as semantic similarity increases. This strictly enforces the theoretically required flatness while preserving effective gradients for response pairs that exhibit genuine quality differences.

\subsection{Safety Mechanism and Reasoning Domain Performance}
A potential concern with incorporating general-purpose semantic encoders into preference optimization is their lack of fine-grained, token-level sensitivity in rigorous domains such as mathematical reasoning or coding. In these fields, a minor syntactic shift (e.g., a single negation sign) can completely alter the underlying utility. \textit{S}-SPPO naturally guards against this risk through its structural safety design: the preference oracle exclusively dictates the optimization direction (i.e., which response wins), while the semantic calibration only conservatively scales the gradient magnitude toward $0.5$. In the worst-case scenario, where the encoder mistakenly identifies two mathematically distinct responses as highly similar, \textit{S}-SPPO merely dampens the training signal, leading to a more conservative optimization update rather than reversing the direction of optimization.

To empirically validate this safety property, we evaluate our model checkpoints on the GSM8K benchmark using strict-match accuracy. As shown in Table~\ref{tab:gsm8k}, \textit{S}-SPPO not only preserves but slightly improves mathematical reasoning capabilities across different base models. This conclusively demonstrates that our semantic calibration does not induce vulnerabilities or performance degradation in reasoning tasks.

\begin{table}[ht]
\centering
\caption{GSM8K Benchmark Performance (Strict-match)}
\label{tab:gsm8k}
\begin{tabular}{lc}
\toprule
\textbf{Model} & \textbf{GSM8K (Strict-match)} \\
\midrule
Mistral-7B Base & 42.8\% \\
Mistral-7B SPPO Iter3 & 44.0\% \\
\textbf{Mistral-7B \textit{S}-SPPO Iter3} & \textbf{44.2\%} \\
\midrule
Llama-3-8B Base & 76.4\% \\
Llama-3-8B SPPO Iter3 & 78.2\% \\
\textbf{Llama-3-8B \textit{S}-SPPO Iter3} & \textbf{78.9\%} \\
\bottomrule
\end{tabular}
\end{table}

\subsection{Scope and Domain-Specific Limitations}
While our empirical results demonstrate that S-SPPO maintains safety on reasoning benchmarks, we explicitly scope our core contribution to general natural language preference alignment tasks. Applying this framework to highly rigorous domains, such as advanced mathematics or complex code generation, introduces distinct challenges that warrant careful consideration. Most off-the-shelf external embedding models are inherently optimized for general natural language semantics. Consequently, they often lack the fine-grained, token-level sensitivity required for these specialized tasks, where a minimal syntactic shift, such as a single negation sign or a modified conditional statement, can drastically alter the underlying semantic meaning and functional utility. Therefore, extending our framework to these specialized domains, potentially through the integration of domain-specific encoders or syntax-aware calibration metrics, remains an important avenue for future investigation.

\subsection{The Necessity of Semantic Calibration with Advanced Reward Models}
While state-of-the-art reward models can partially mitigate calibration issues, frameworks like \textit{S}-SPPO remain fundamentally necessary for real-world alignment pipelines due to three primary considerations. First, regarding the inherent gap to a perfect oracle, current reward models still exhibit overconfidence and calibration issues when evaluating semantically similar responses. A theoretically flawless preference oracle is hypothetical; in practice, human preferences are intrinsically noisy, making explicit semantic calibration a necessary structural safeguard against persistent miscalibration. Second, \textit{S}-SPPO inherently aligns with the premise of the SPPO paradigm. The fundamental motivation behind Self-Play Preference Optimization is that human preferences are profoundly challenging to model perfectly. If a flawless oracle existed, algorithmic distinctions between various advanced preference optimization methods would largely diminish. \textit{S}-SPPO is designed specifically to address and correct these inevitable imperfections of practical reward models. Finally, computational efficiency in iterative self-play remains a critical bottleneck. Iterative self-play requires generating $K$ candidates and performing $O(K^2)$ pairwise evaluations per prompt. Scaling this quadratic complexity with massive state-of-the-art models, such as 100B parameter models or LLM as a judge pipelines, over extensive datasets is computationally prohibitive. \textit{S}-SPPO resolves this performance and computation tradeoff by enabling highly efficient and lightweight reward models, such as the 0.4B parameter PairRM, to achieve robust alignment, making extensive self-play scaling computationally feasible without astronomical inference costs.

\subsection{Future Work}
While \textit{S}-SPPO demonstrates the effectiveness of dual-space semantic calibration, our exploration opens several avenues for future research. First, regarding the regularization weight $\lambda$, our current implementation maintains a fixed value across all iterations to demonstrate the robustness of the method without extensive hyperparameter tuning. However, it is plausible that a dynamic schedule of $\lambda$ could further refine the equilibrium state (we observed some gain on this), though we prioritized simplicity and establishing a proof-of-concept over heavy tuning in this work. Second, we present a specific instantiation of calibration via linear semantic gating and Euclidean latent repulsion, yet this is merely one possibility within a broader design space. Future work could explore alternative calibration mechanisms, such as non-linear gating functions, uncertainty-aware calibration leveraging epistemic uncertainty, or kernel-based measures to capture more complex semantic dependencies. Finally, our current framework utilizes a fixed, off-the-shelf semantic encoder to guide the policy's latent geometry. A promising direction would be to more explicitly model the relationship between the external semantic space and the model's internal representation, potentially by learning a projection that aligns the encoder's embedding manifold directly with the policy's latent space or by fine-tuning the encoder end-to-end to better capture the nuances of the specific task distribution.

%% file: tables/aba-2.tex

\begin{table}[h]
\centering
\caption{Computational efficiency comparison between SPPO and S-SPPO on Llama-3-8B per iteration. The \textbf{Total} runtime includes the dominant \textit{Generation \& Annotation} phases (Steps 3--5 in Alg.~\ref{alg:s_sppo}), which are identical for both methods ($\sim$49 mins).}
\resizebox{0.5\textwidth}{!}{%
\begin{tabular}{lccc}
\toprule
{Phase} & {\makecell{Preference\\Construction}} & {\makecell{Policy\\Optimization}} & {\makecell{Total\\Pipeline}} \\ 
\cmidrule(lr){2-2} \cmidrule(lr){3-3} \cmidrule(lr){4-4} 
\textit{Alg. Steps} & \textit{(Steps 5--7)} & \textit{(Step 8)} & \textit{(Steps 3--7)} \\ 
\midrule
\multicolumn{4}{l}{\textit{\textbf{Runtime}}} \\
SPPO (Baseline) & $\sim$ 90s & $\sim$ 16 min & $\sim$ 66 min \\
\textit{S}-SPPO (Ours) & $\sim$ 95s & $\sim$ 16 min & $\sim$ 66 min \\ 
\midrule
\multicolumn{4}{l}{\textit{\textbf{Peak GPU Memory}}} \\
SPPO (Baseline) & $\sim$ 0 GB & 67.0 GB & -- \\
\textit{S}-SPPO (Ours) & $\sim$ 1.7 GB & 72.8 GB & -- \\ 
\bottomrule
\end{tabular}%
}\label{tab:efficiency}
\end{table}

%% file: example_paper.bib
@article{ouyang2022training,
  title={Training language models to follow instructions with human feedback},
  author={Ouyang, Long and Wu, Jeffrey and Jiang, Xu and Almeida, Diogo and Wainwright, Carroll and Mishkin, Pamela and Zhang, Chong and Agarwal, Sandhini and Slama, Katarina and Ray, Alex and others},
  journal={Advances in neural information processing systems},
  volume={35},
  pages={27730--27744},
  year={2022}
}

@article{bai2022training,
  title={Training a helpful and harmless assistant with reinforcement learning from human feedback},
  author={Bai, Yuntao and Jones, Andy and Ndousse, Kamal and Askell, Amanda and Chen, Anna and DasSarma, Nova and Drain, Dawn and Fort, Stanislav and Ganguli, Deep and Henighan, Tom and others},
  journal={arXiv preprint arXiv:2204.05862},
  year={2022}
}

@article{wang2020minilm,
  title={Minilm: Deep self-attention distillation for task-agnostic compression of pre-trained transformers},
  author={Wang, Wenhui and Wei, Furu and Dong, Li and Bao, Hangbo and Yang, Nan and Zhou, Ming},
  journal={Advances in neural information processing systems},
  volume={33},
  pages={5776--5788},
  year={2020}
}

@article{touvron2023llama,
  title={Llama: Open and efficient foundation language models},
  author={Touvron, Hugo and Lavril, Thibaut and Izacard, Gautier and Martinet, Xavier and Lachaux, Marie-Anne and Lacroix, Timoth{\'e}e and Rozi{\`e}re, Baptiste and Goyal, Naman and Hambro, Eric and Azhar, Faisal and others},
  journal={arXiv preprint arXiv:2302.13971},
  year={2023}
}

@article{nussbaum2024nomic,
  title={Nomic embed: Training a reproducible long context text embedder},
  author={Nussbaum, Zach and Morris, John X and Duderstadt, Brandon and Mulyar, Andriy},
  journal={arXiv preprint arXiv:2402.01613},
  year={2024}
}

@article{gunther2023jina,
  title={Jina embeddings 2: 8192-token general-purpose text embeddings for long documents},
  author={G{\"u}nther, Michael and Ong, Jackmin and Mohr, Isabelle and Abdessalem, Alaeddine and Abel, Tanguy and Akram, Mohammad Kalim and Guzman, Susana and Mastrapas, Georgios and Sturua, Saba and Wang, Bo and others},
  journal={arXiv preprint arXiv:2310.19923},
  year={2023}
}

@inproceedings{sun2024aligning,
  title={Aligning large multimodal models with factually augmented rlhf},
  author={Sun, Zhiqing and Shen, Sheng and Cao, Shengcao and Liu, Haotian and Li, Chunyuan and Shen, Yikang and Gan, Chuang and Gui, Liangyan and Wang, Yu-Xiong and Yang, Yiming and others},
  booktitle={Findings of the Association for Computational Linguistics: ACL 2024},
  pages={13088--13110},
  year={2024}
}

@article{chen2025aha,
  title={AHA: Aligning Large Audio-Language Models for Reasoning Hallucinations via Counterfactual Hard Negatives},
  author={Chen, Yanxi and Zhu, Wenhui and Chen, Xiwen and Wang, Zhipeng and Li, Xin and Qiu, Peijie and Wang, Hao and Dong, Xuanzhao and Xiong, Yujian and Schneider, Anderson and others},
  journal={arXiv preprint arXiv:2512.24052},
  year={2025}
}

@article{christiano2017deep,
  title={Deep reinforcement learning from human preferences},
  author={Christiano, Paul F and Leike, Jan and Brown, Tom and Martic, Miljan and Legg, Shane and Amodei, Dario},
  journal={Advances in neural information processing systems},
  volume={30},
  year={2017}
}

@article{schulman2017proximal,
  title={Proximal policy optimization algorithms},
  author={Schulman, John and Wolski, Filip and Dhariwal, Prafulla and Radford, Alec and Klimov, Oleg},
  journal={arXiv preprint arXiv:1707.06347},
  year={2017}
}

@inproceedings{azar2024general,
  title={A general theoretical paradigm to understand learning from human preferences},
  author={Azar, Mohammad Gheshlaghi and Guo, Zhaohan Daniel and Piot, Bilal and Munos, Remi and Rowland, Mark and Valko, Michal and Calandriello, Daniele},
  booktitle={International Conference on Artificial Intelligence and Statistics},
  pages={4447--4455},
  year={2024},
  organization={PMLR}
}

@article{rafailov2023direct,
  title={Direct preference optimization: Your language model is secretly a reward model},
  author={Rafailov, Rafael and Sharma, Archit and Mitchell, Eric and Manning, Christopher D and Ermon, Stefano and Finn, Chelsea},
  journal={Advances in neural information processing systems},
  volume={36},
  pages={53728--53741},
  year={2023}
}

@article{bradley1952rank,
  title={Rank analysis of incomplete block designs: I. the method of paired comparisons},
  author={Bradley, Ralph Allan and Terry, Milton E},
  journal={Biometrika},
  volume={39},
  number={3/4},
  pages={324--345},
  year={1952},
  publisher={JSTOR}
}

@article{dubois2024alpacafarm,
  title={Alpacafarm: A simulation framework for methods that learn from human feedback},
  author={Dubois, Yann and Li, Chen Xuechen and Taori, Rohan and Zhang, Tianyi and Gulrajani, Ishaan and Ba, Jimmy and Guestrin, Carlos and Liang, Percy S and Hashimoto, Tatsunori B},
  journal={Advances in Neural Information Processing Systems},
  volume={36},
  pages={30039--30069},
  year={2023}
}

@inproceedings{cui2023ultrafeedback,
  title={ULTRAFEEDBACK: Boosting Language Models with Scaled AI Feedback},
  author={Cui, Ganqu and Yuan, Lifan and Ding, Ning and Yao, Guanming and He, Bingxiang and Zhu, Wei and Ni, Yuan and Xie, Guotong and Xie, Ruobing and Lin, Yankai and others},
  booktitle={Forty-first International Conference on Machine Learning}
}

@misc{jiang2023mistral,
      title={Mistral 7B}, 
      author={Albert Q. Jiang and Alexandre Sablayrolles and Arthur Mensch and Chris Bamford and Devendra Singh Chaplot and Diego de las Casas and Florian Bressand and Gianna Lengyel and Guillaume Lample and Lucile Saulnier and Lélio Renard Lavaud and Marie-Anne Lachaux and Pierre Stock and Teven Le Scao and Thibaut Lavril and Thomas Wang and Timothée Lacroix and William El Sayed},
      year={2023},
      eprint={2310.06825},
      archivePrefix={arXiv},
      primaryClass={cs.CL},
      url={https://arxiv.org/abs/2310.06825}, 
}

@article{sturua2024jina,
  title={jina-embeddings-v3: Multilingual embeddings with task lora},
  author={Sturua, Saba and Mohr, Isabelle and Akram, Mohammad Kalim and G{\"u}nther, Michael and Wang, Bo and Krimmel, Markus and Wang, Feng and Mastrapas, Georgios and Koukounas, Andreas and Wang, Nan and others},
  journal={arXiv preprint arXiv:2409.10173},
  year={2024}
}

@article{freund1999adaptive,
  title={Adaptive game playing using multiplicative weights},
  author={Freund, Yoav and Schapire, Robert E},
  journal={Games and Economic Behavior},
  volume={29},
  number={1-2},
  pages={79--103},
  year={1999},
  publisher={Elsevier}
}

@inproceedings{
wu2024self,
title={Self-Play Preference Optimization for Language Model Alignment},
author={Yue Wu and Zhiqing Sun and Huizhuo Yuan and Kaixuan Ji and Yiming Yang and Quanquan Gu},
booktitle={The Thirteenth International Conference on Learning Representations},
year={2025},
url={https://openreview.net/forum?id=a3PmRgAB5T}
}

@article{zhao2023slic,
  title={Slic-hf: Sequence likelihood calibration with human feedback},
  author={Zhao, Yao and Joshi, Rishabh and Liu, Tianqi and Khalman, Misha and Saleh, Mohammad and Liu, Peter J},
  journal={arXiv preprint arXiv:2305.10425},
  year={2023}
}

@article{ethayarajh2024kto,
  title={Kto: Model alignment as prospect theoretic optimization},
  author={Ethayarajh, Kawin and Xu, Winnie and Muennighoff, Niklas and Jurafsky, Dan and Kiela, Douwe},
  journal={arXiv preprint arXiv:2402.01306},
  year={2024}
}

@inproceedings{liu2024statistical,
  title={Statistical rejection sampling improves preference optimization},
  author={Liu, Tianqi and Zhao, Yao and Joshi, Rishabh and Khalman, Misha and Saleh, Mohammad and Liu, Peter and Liu, Jialu},
  booktitle={International conference on learning representations},
  volume={2024},
  pages={54605--54624},
  year={2024}
}

@article{meng2024simpo,
  title={Simpo: Simple preference optimization with a reference-free reward},
  author={Meng, Yu and Xia, Mengzhou and Chen, Danqi},
  journal={Advances in Neural Information Processing Systems},
  volume={37},
  pages={124198--124235},
  year={2024}
}

@article{li2020tilted,
  title={Tilted empirical risk minimization},
  author={Li, Tian and Beirami, Ahmad and Sanjabi, Maziar and Smith, Virginia},
  journal={arXiv preprint arXiv:2007.01162},
  year={2020}
}

@inproceedings{kim2025sdpo,
  title={sdpo: Don’t use your data all at once},
  author={Kim, Dahyun and Kim, Yungi and Song, Wonho and Kim, Hyeonwoo and Kim, Yunsu and Kim, Sanghoon and Park, Chanjun},
  booktitle={Proceedings of the 31st International Conference on Computational Linguistics: Industry Track},
  pages={366--373},
  year={2025}
}

@article{rosset2024direct,
  title={Direct nash optimization: Teaching language models to self-improve with general preferences},
  author={Rosset, Corby and Cheng, Ching-An and Mitra, Arindam and Santacroce, Michael and Awadallah, Ahmed and Xie, Tengyang},
  journal={arXiv preprint arXiv:2404.03715},
  year={2024}
}

@inproceedings{yuan2024self,
  title={Self-rewarding language models},
  author={Yuan, Weizhe and Pang, Richard Yuanzhe and Cho, Kyunghyun and Li, Xian and Sukhbaatar, Sainbayar and Xu, Jing and Weston, Jason E},
  booktitle={Forty-first International Conference on Machine Learning},
  year={2024}
}

@article{dubois2024length,
  title={Length-controlled alpacaeval: A simple way to debias automatic evaluators},
  author={Dubois, Yann and Galambosi, Bal{\'a}zs and Liang, Percy and Hashimoto, Tatsunori B},
  journal={arXiv preprint arXiv:2404.04475},
  year={2024}
}

@inproceedings{
li2024crowdsourced,
title={From Crowdsourced Data to High-quality Benchmarks: Arena-Hard and Benchbuilder Pipeline},
author={Tianle Li and Wei-Lin Chiang and Evan Frick and Lisa Dunlap and Tianhao Wu and Banghua Zhu and Joseph E. Gonzalez and Ion Stoica},
booktitle={Forty-second International Conference on Machine Learning},
year={2025},
url={https://openreview.net/forum?id=KfTf9vFvSn}
}

@article{zheng2024judging,
  title={Judging llm-as-a-judge with mt-bench and chatbot arena},
  author={Zheng, Lianmin and Chiang, Wei-Lin and Sheng, Ying and Zhuang, Siyuan and Wu, Zhanghao and Zhuang, Yonghao and Lin, Zi and Li, Zhuohan and Li, Dacheng and Xing, Eric and others},
  journal={Advances in neural information processing systems},
  volume={36},
  pages={46595--46623},
  year={2023}
}

@inproceedings{pasztor2025stackelberg,
  title={Stackelberg Learning from Human Feedback: Preference Optimization as a Sequential Game},
  author={P{\'a}sztor, Barna and Buening, Thomas Kleine and Krause, Andreas},
  booktitle={NeurIPS 2025 Workshop: Second Workshop on Aligning Reinforcement Learning Experimentalists and Theorists},
  year={2025}
}

@article{tang2025game,
  title={Game-Theoretic Regularized Self-Play Alignment of Large Language Models},
  author={Tang, Xiaohang and Yoon, Sangwoong and Son, Seongho and Yuan, Huizhuo and Gu, Quanquan and Bogunovic, Ilija},
  journal={arXiv preprint arXiv:2503.00030},
  year={2025}
}

@article{ji2024self,
  title={Self-play with adversarial critic: Provable and scalable offline alignment for language models},
  author={Ji, Xiang and Kulkarni, Sanjeev and Wang, Mengdi and Xie, Tengyang},
  journal={arXiv preprint arXiv:2406.04274},
  year={2024}
}

@inproceedings{le2025token,
  title={Token-Level Self-Play with Importance-Aware Guidance for Large Language Models},
  author={Le, Tue and Vuong, Hoang Tran and Tran, Quyen and Van, Linh Ngo and Harandi, Mehrtash and Le, Trung},
  booktitle={The Thirty-ninth Annual Conference on Neural Information Processing Systems},
  year={2025}
}

@inproceedings{xiang2025self,
  title={Self-steering optimization: Autonomous preference optimization for large language models},
  author={Xiang, Hao and Yu, Bowen and Lin, Hongyu and Lu, Keming and Lu, Yaojie and Han, Xianpei and He, Ben and Sun, Le and Zhou, Jingren and Lin, Junyang},
  booktitle={Findings of the Association for Computational Linguistics: ACL 2025},
  pages={9073--9085},
  year={2025}
}

@inproceedings{lee2025flat,
  title={Flat reward in policy parameter space implies robust reinforcement learning},
  author={Lee, Hyun Kyu and Yoon, Sung Whan},
  booktitle={The Thirteenth International Conference on Learning Representations},
  year={2025}
}

@article{casper2023open,
  title={Open Problems and Fundamental Limitations of Reinforcement Learning from Human Feedback},
  author={Casper, Stephen and Davies, Xander and Shi, Claudia and Krendl Gilbert, Thomas and Scheurer, J{\'e}r{\'e}my and Rando Ramirez, Javier and Freedman, Rachel and Korbak, Tomasz and Lindner, David and Freire, Pedro and others},
  journal={Transactions on Machine Learning Research},
  year={2023},
  publisher={OpenReview}
}

@inproceedings{chen2024self,
  title={Self-play fine-tuning convertsweak language models to strong language models},
  author={Chen, Zixiang and Deng, Yihe and Yuan, Huizhuo and Ji, Kaixuan and Gu, Quanquan},
  booktitle={Proceedings of the 41st International Conference on Machine Learning},
  pages={6621--6642},
  year={2024}
}

@article{dong2024rlhf,
  title={RLHF Workflow: From Reward Modeling to Online RLHF A Comprehensive Practical Alignment Recipe of Iterative Preference Learning},
  author={Dong, Hanze and Xiong, Wei and Pang, Bo and Wang, Haoxiang and Zhao, Han and Zhou, Yingbo and Jiang, Nan and Sahoo, Doyen and Xiong, Caiming and Zhang, Tong},
  journal={Transactions on Machine Learning Research},
  volume={2024},
  year={2024},
  publisher={Transactions on Machine Learning Research}
}

@inproceedings{engstrom2020implementation,
  title={Implementation Matters in Deep Policy Gradients: A Case Study on PPO and TRPO},
  author={Engstrom, Logan and Ilyas, Andrew and Santurkar, Shibani and Tsipras, Dimitris and Janoos, Firdaus and Rudolph, Larry and Madry, Aleksander},
  booktitle={International Conference on Learning Representations},
  year={2020}
}

@inproceedings{hong2024orpo,
  title={Orpo: Monolithic preference optimization without reference model},
  author={Hong, Jiwoo and Lee, Noah and Thorne, James},
  booktitle={Proceedings of the 2024 Conference on Empirical Methods in Natural Language Processing},
  pages={11170--11189},
  year={2024}
}

@inproceedings{jiang2023llm,
  title={Llm-blender: Ensembling large language models with pairwise ranking and generative fusion},
  author={Jiang, Dongfu and Ren, Xiang and Lin, Bill Yuchen},
  booktitle={Proceedings of the 61st Annual Meeting of the Association for Computational Linguistics (Volume 1: Long Papers)},
  pages={14165--14178},
  year={2023}
}

@inproceedings{wang2025magnetic,
  title={Magnetic preference optimization: Achieving last-iterate convergence for language model alignment},
  author={Wang, Mingzhi and Ma, Chengdong and Chen, Qizhi and Meng, Linjian and Han, Yang and Xiao, Jiancong and Zhang, Zhaowei and Huo, Jing and Su, Weijie and Yang, Yaodong},
  booktitle={International Conference on Learning Representations},
  volume={2025},
  pages={35759--35790},
  year={2025}
}

@inproceedings{xiong2024iterative,
  title={Iterative Preference Learning from Human Feedback: Bridging Theory and Practice for RLHF under KL-constraint},
  author={Xiong, Wei and Dong, Hanze and Ye, Chenlu and Wang, Ziqi and Zhong, Han and Ji, Heng and Jiang, Nan and Zhang, Tong},
  booktitle={Forty-first International Conference on Machine Learning},
  year={2024}
}
